\documentclass[a4paper]{article}

\usepackage{amsmath}
\usepackage{amssymb}
\usepackage{amsthm}
\usepackage[dvipdfmx]{graphicx}
\usepackage{url}
\usepackage{mathbbol}
\usepackage{amssymb}

\usepackage[utf8]{inputenc} 
\usepackage[T1]{fontenc}    
\usepackage{hyperref}       
\usepackage{url}            
\usepackage{booktabs}       
\usepackage{amsfonts}       
\usepackage{nicefrac}       
\usepackage{microtype}      
\usepackage{multirow}
\usepackage{tabularx}
\usepackage{booktabs}
\usepackage{caption}
\usepackage{subcaption}

\DeclareSymbolFontAlphabet{\amsmathbb}{AMSb}%
\newcommand{\R}{\amsmathbb{R}}

\renewcommand{\Pr}{\mathsf{Pr}}

\newcommand{\prob}{\mathrm{Prob}}
\newcommand{\dirac}{\mathbf{d}}

\newcommand{\ddist}[3]{\Delta_{#3}(#1 || #2)}

\newtheorem{theorem}{Theorem}
\newtheorem{lemma}[theorem]{Lemma}
\newtheorem{corollary}[theorem]{Corollary}
\newtheorem{definition}[theorem]{Definition}

\newcommand{\supp}{\mathsf{supp}}

\newcommand{\inverse}[1]{{#1}^{\hspace{-0.15em} - \hspace{-0.1em}1}\hspace{-0.1em}}
\usepackage[textsize=footnotesize,color=green!40]{todonotes}

\newcommand{\stareq}{\stackrel{\mathsf{(*)}}{=}}
\newcommand{\defeq}{\stackrel{\mathsf{def}}{=}}
\newcommand{\EE}{\Bbb E}

\usepackage{mathpartir}
\usepackage{braket}
\usepackage{comment}
\usepackage{txfonts}

\newcommand{\argmax}{\mathop{\rm argmax}\limits}
\newcommand{\argmin}{\mathop{\rm argmin}\limits}

\begin{document}
\title{Hypothesis Testing Interpretations and \\R\'enyi Differential Privacy \\(Long version)}
\author{Borja Balle \and Gilles Barthe \and Marco Gaboardi \and Justin Hsu \and Tetsuya Sato}
\maketitle

\begin{abstract}
  Differential privacy is a de facto standard in data privacy, with applications
  in the public and private sectors. A way to explain differential
  privacy, which is particularly appealing to statistician and social
  scientists,  is by means of its statistical hypothesis testing
  interpretation. Informally, one cannot
  effectively test whether a specific individual has contributed her data by
  observing the output of a private mechanism---any test cannot have both high
  significance and high power.

  In this paper, we identify some conditions under which a privacy
  definition given in terms of a statistical divergence satisfies a
  similar interpretation. These conditions are useful 
  to analyze the distinguishability power of divergences and we use
  them to study the hypothesis testing interpretation of some relaxations of differential
  privacy based on \emph{R\'enyi divergence}. This analysis also
  results in an improved conversion
  rule between these definitions and differential privacy.
\end{abstract}

\section{Introduction}
Differential privacy~\cite{DworkMcSherryNissimSmith2006} is a formal notion of
data privacy that enables accurate statistical analyses on populations while
preserving privacy for individuals contributing their data. Differential privacy
is supported by a rich theory, with sophisticated algorithms for common data
analysis tasks and composition theorems to simplify the design and formal
analysis of new private algorithms. This theory has helped make differential
privacy a de facto standard for privacy-preserving data analysis. Over the last
years, differential privacy has found use in the private
sector~\cite{KenthapadiMT19} by companies such as
Google~\cite{ErlingssonPK14,PapernotSMRTE18}, Apple~\cite{apple}, and
Uber~\cite{JohnsonNS18}, and in the public sector by agencies such as the U.S.
Census Bureau~\cite{Abowd18,GarfinkelAP18}.

A common challenge faced across all uses of differential privacy is to explain
its guarantees to users and policy makers. Indeed, differential privacy first
emerged in the theoretical computer science community, and was only
subsequently considered in other research areas interested in data privacy.  For
this reason, several works have attempted to provide different interpretations
of the \emph{semantics} of differential privacy in an effort to make it more
accessible.

One approach that has been particularly successful, especially when
introducing differential privacy to people versed in statistical data
analysis, is the \emph{hypothesis testing} interpretation of differential
privacy~\cite{WassermanZ10,KairouzOV15}. One can imagine an
experiment where one wants to test, based on the output of a differentially private
mechanism, the \emph{null hypothesis} that an individual $I$ has contributed her
data to a particular dataset $x_0$. One can also imagine that an
\emph{alternative hypothesis} is that the individual $I$ has not contributed her
data. Then, the definition of differential privacy guarantees---and is in fact
equivalent to requiring---that every hypothesis test has either low
\emph{significance} (it has a high rate of Type I errors), or low \emph{power}
(it has a high rate of Type II errors). Under this interpretation, the privacy
parameters $(\epsilon, \delta)$ control the tradeoff between significance and
power.

Recently, several variants of differential privacy have been
proposed~\cite{dwork2016concentrated,BunS16,Mironov17,BDRS18,DongRS19}. Most of these new
privacy definitions have been proposed as privacy notions with better
composition properties than differential privacy. Having better
composition can become a key advantage when a high number of data
accesses is needed for a single analysis (e.g., in private deep
learning~\cite{AbadiCGMMT016}). Technically, many of these variants are
formulated as bounds on the \emph{R\'enyi divergence} between the distribution
obtained when running a private mechanism over a dataset where an individual $I$
has contributed her data versus the case when the private mechanism is run over
the dataset where $I$'s data is removed. 

In this work we develop some analytical tools to study the hypothesis
testing interpretation of privacy definitions based on statistical
divergences. The first notion we introduce is the concept of  \emph{$k$-cut} of
a divergence. Intuitively this notion corresponds to restricting
the distributions that are input to a divergence to a finite domain of
cardinality $k$. We can think about the functions implementing the
restrictions as (probabilistic) decision rules with $k$ possible outcomes.
The second notion we introduce is the concept of \emph{$k$-generatedness} for
a divergence. Intuitively, a divergence is $k$-generated if it is
equal to its $k$-cut. This notion expresses the number of decisions that
are needed in decision rules to fully characterize the divergence.

We use these two analytical tools to show that a privacy definition
based on a divergence has an hypothesis testing interpretation if and
only if it is $2$-generated. We show that the divergence
characterizing differential privacy is indeed $2$-generated and that
the notion of $2$-generatedness corresponds to the notion of privacy
regions introduced in~\cite{KairouzOV15}. On the negative side we show formally
that variants of differential privacy based on the R\'enyi divergence
do not admit directly a hypothesis testing interpretation because the
R\'enyi divergence is exactly $\infty$-generated (where by $\infty$ we mean that it is infinitely, but countably generated). Nevertheless, we
show that one can achieve the hypothesis testing interpretation by
considering the $2$-cut of the R\'enyi divergence. Intuitively, this says that to characterize relaxations of differential
privacy based on the R\'enyi divergence through an experiment similar to the one used in the hypothesis testing
interpretation, one needs either to restrict the distinguishability
power of the divergence or to consider an infinite number of possible hypothesis.
This shows a semantics separation between standard differential privacy and
relaxations based on R\'enyi divergence.



In addition, we use the analytical tools we develop to study the
relations between different privacy definitions. Specifically, we use
the $2$-cut of R\'enyi divergence to give better conversion rules from
R\'enyi differential privacy to $(\epsilon,\delta)$-differential
privacy, and to study the relations with \emph{Gaussian Differential
  Privacy}~\cite{DongRS19} another formal definition of
privacy inspired by the hypothesis testing interpretation which was
recently proposed. 

Finally, we study a sufficient condition to guarantee that a divergence
is $k$-generated: divergences defined as a supremum of a quasi-convex function
$F$ over probabilities of $k$-partitions are $k$-generated. This allows one to
construct divergences supporting the hypothesis testing interpretation by
requiring them to be defined through an function $F$ giving a $2$-generated divergence.
The condition is also necessary for quasi-convex divergences, characterizing
$k$-generation for all quasi-convex divergences.

Summarizing, our contributions are:
\begin{list}{{(\arabic{enumi})}}{\usecounter{enumi}
\setlength{\leftmargin}{11pt}
\setlength{\listparindent}{-1pt}
\setlength{\parsep}{-1pt}}
\vspace*{-1ex}
\item We first introduce the notions of \emph{$k$-cut} and \emph{$k$-generatedness} for
  divergences. These notions allow one to measure the power of
  divergences in terms of the number of possible decisions that are
  needed in a test to fully characterize the divergence.
\item We show that the divergence used to characterize differential privacy is
  $2$-generated, supporting the usual hypothesis testing interpretation of
  differential privacy
\item We show that R\'enyi divergence is $\infty$-generated, ruling
  out a direct   hypothesis testing interpretation for privacy notions
  based on it. Nevertheless, we show that one can achieve the
  hypothesis testing interpretation by considering the $2$-cut of
  R\'enyi divergence.
\item We use our analytic tools to study other notions of privacy and
  to give better conversion rules between R\'enyi differential privacy
  and $(\epsilon,\delta)$-differential privacy.
\item We give sufficient and necessary conditions for a quasi-convex divergence
  to be $k$-generated.
\end{list}
\section{Background: hypothesis testing, privacy, and R\'enyi divergences}
\subsection{Hypothesis testing interpretation for $(\varepsilon,\delta)$-differential privacy}

We view \emph{randomized algorithms} as functions
$\mathcal{M} \colon X \to \prob(Y)$ from a set $X$ of inputs to the
set $\prob(Y)$ of discrete \emph{probability distributions} over a
set $Y$ of outputs. We assume that $X$ is equipped with a symmetric
\emph{adjacency relation}---informally, inputs are datasets and two inputs
$x_0$ and $x_1$ are adjacent iff they differ in the data of a single
individual.
\begin{definition}[Differential Privacy (DP) \cite{DworkMcSherryNissimSmith2006}]
  Let $\varepsilon>0$ and $0\leq \delta \leq 1$. A randomized
  algorithm $\mathcal{M} \colon X \to \prob(Y)$ is
  \emph{$(\varepsilon,\delta)$-differentially private} if for every
  pairs of adjacent inputs $x_0$ and $x_1$, and every subset
  $S\subseteq Y$, we have:
  $$
   \Pr[\mathcal{M}(x_0) \in S] \leq e^\varepsilon \Pr[\mathcal{M}(x_1) \in S] + \delta.
  $$
\end{definition}

\cite{WassermanZ10,KairouzOV15} proposed a useful interpretation of this
guarantee in terms of \emph{hypothesis testing}. Suppose that $x_0$ and $x_1$
are adjacent inputs. The observer sees the output $y$ of running a private
mechanism $\mathcal{M}$ on one of these inputs---but does not see the particular
input---and wants to guess whether the input was $x_0$ or $x_1$. 

In the terminology of hypothesis testing, let $y \in Y$ be an output
of a randomized mechanism $\mathcal{M}$, and take the following \emph{null} and
\emph{alternative} hypotheses:
\begin{center}
  {\bf H0} : $y$ came from $\mathcal{M}(x_0)$, \qquad\qquad  {\bf H1} : $y$ came from $\mathcal{M}(x_1)$.
\end{center}
One simple way of deciding between the two hypotheses is to fix a
\emph{rejection region} $S \subseteq Y$; if the observation $y$ is in $S$ then
the null hypothesis is rejected, and if the
observation $y$ is not in $S$ then the null hypothesis is not
rejected. This is an example of a \emph{deterministic decision rule}.

Each decision rule can err in two possible ways. A \emph{false alarm} (i.e.\ Type I
error) is when the null hypothesis is true but rejected. This error rate is
defined as 
${\tt PFA}(x_0,x_1,\mathcal{M},S) \defeq \Pr[\mathcal{M}(x_0) \in S]$.
On the other hand, the decision rule may incorrectly fail to reject the null
hypothesis, a \emph{false negative} (i.e.\ Type II error). The probability of missed
detection is defined as
${\tt PMD}(x_0,x_1,\mathcal{M},S) \defeq \Pr[\mathcal{M}(x_1) \notin S]$. There
is a natural tradeoff between these two errors---a rule with a larger rejection
region will be less likely to incorrectly fail to reject but more likely to
incorrectly reject, while a rule with a smaller rejection region will be less
likely to incorrectly reject but more likely to incorrectly fail to reject.

Differential privacy can now be reformulated in terms of these error rates.

\begin{theorem}[\cite{WassermanZ10,KairouzOV15}] \label{thm:dp-hyp-test}
  A randomized algorithm $\mathcal{M} \colon X \to \prob(Y)$ is
  $(\varepsilon,\delta)$-differentially private if and only if for
  every pair of adjacent inputs $x_0$ and $x_1$, and any rejection region
  $S\subseteq Y$, we have:
  $ {\tt PFA}(x_0,x_1,\mathcal{M},S)+e^\varepsilon {\tt
    PMD}(x_0,x_1,\mathcal{M},S)\geq 1-\delta $ and
  $e^\varepsilon{\tt PFA}(x_0,x_1,\mathcal{M},S)+{\tt
    PMD}(x_0,x_1,\mathcal{M},S)\geq 1-\delta $.
\end{theorem}

Intuitively, the lower bound on the sum of the two error rates means that
\emph{no} decision rule is capable of achieving low Type I error and low Type II
error simultaneously. Thus, the output distributions from any two adjacent
inputs are statistically hard to distinguish.

Following \cite{KairouzOV15}, we can also reformulate the definition
of differential privacy in terms of a \emph{privacy region} describing the
attainable pairs of Type I and Type II errors.
\begin{theorem}[\cite{KairouzOV15}]
A randomized algorithm $\mathcal{M} \colon X \to \prob(Y)$
is $(\varepsilon,\delta)$-differentially private if and only if
for every pair of adjacent inputs $x_0$ and $x_1$, and every rejection region $S
\subseteq Y$, we have
\[
({\tt PFA}(x_0,x_1,\mathcal{M},S),
{\tt PMD}(x_0,x_1,\mathcal{M},S)) \in R(\varepsilon,\delta) ,
\]
where the privacy region $R(\varepsilon,\delta)$ is defined as:
\[
R(\varepsilon,\delta) 
= \Set{ (x,y)  \in [0,1] \times [0,1]  | (1 - x) \leq e^\varepsilon y + \delta}.
\]
\end{theorem}

Figure \ref{Fig:PrivacyregionofRR} shows an example of a privacy
region  (the area between the dashed lines) $R(0.67,0.05)$ and its mirror image and of all the points $({\tt PFA},{\tt PMD})$ 
that can be generated by a randomized response mechanism
$\mathcal{M}_{\mathtt{RR}} \colon \{0,1\}^3 \to \prob(\{0,1\}^3)$
working on vectors of three bits and flipping each bit with probability $0.34$.

\begin{figure}
  \centering
  \includegraphics[height = 7cm, width = 7cm]{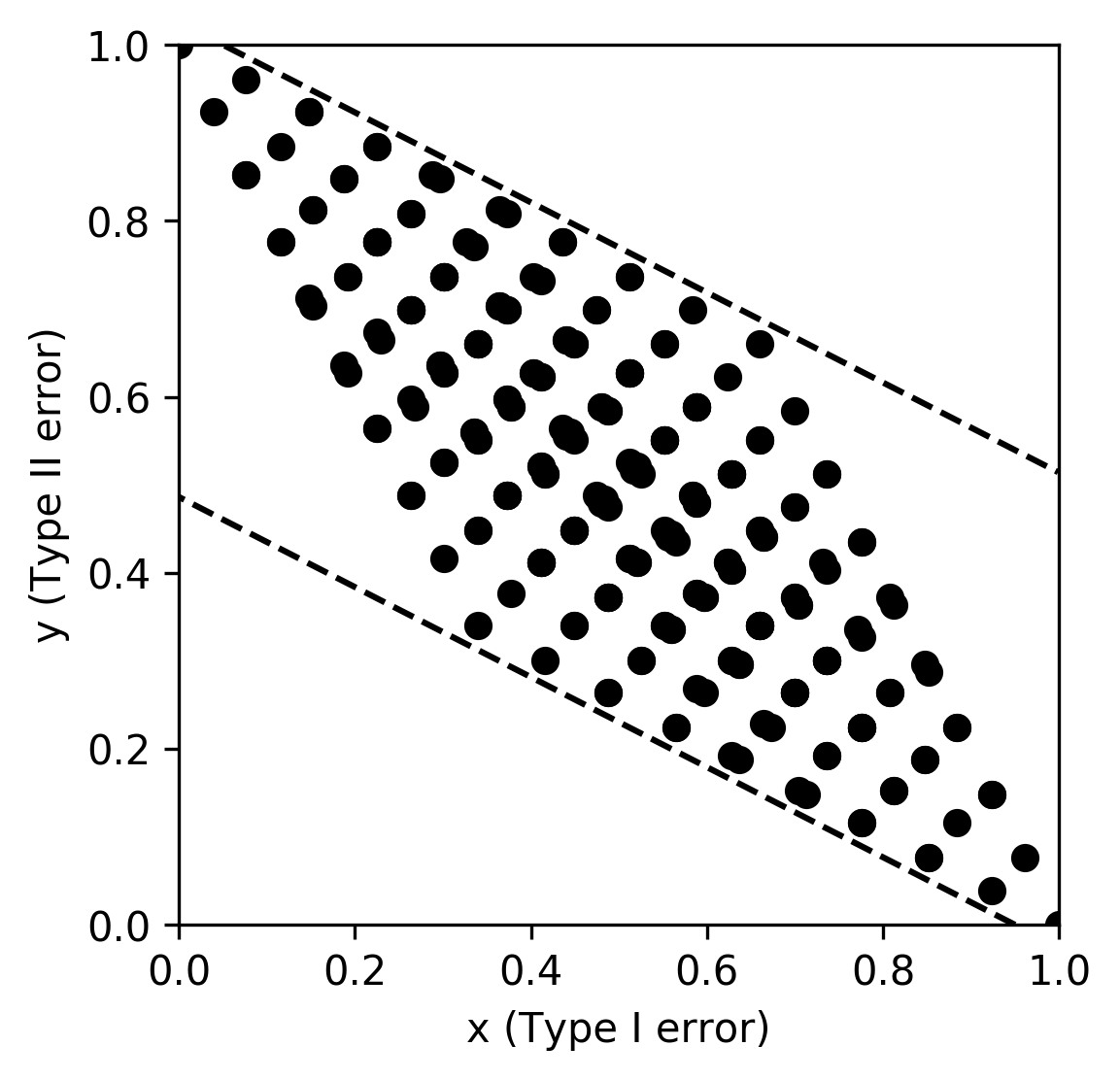}
  \caption{Pairs $({\tt PFA},{\tt PMD})$ of $\mathcal{M}_{\mathtt{RR}}$ and $R(0.67,0.05)$}
\label{Fig:PrivacyregionofRR}
\end{figure}

Since the original introduction of differential privacy, researchers have
proposed several other variants based on R\'enyi divergence. The central
question of this paper is: can we give similar hypothesis testing
interpretations to these (and other) variants of differential privacy? 

\subsection{Variants of differential privacy based on R\'enyi divergence}

We recall here notions of differential privacy based on R\'enyi divergence.

\begin{definition}[R\'enyi divergence~\cite{Renyi1961}]
  Let $\alpha > 1$. The \emph{R\'enyi divergence} of order $\alpha$
  between two probability distributions $\mu_1$ and $\mu_2$ on a space $X$ is defined by:
\begin{equation}\label{eq:Renyi_divergences}
  D^{\alpha}_X(\mu_1||\mu_2) \defeq \frac{1}{\alpha - 1} \log
  \sum_{x\in X} \mu_2(x) \left( \frac{\mu_1(x)}{\mu_2(x)} \right)^\alpha.
\end{equation}
\end{definition}
The above definition does not consider the cases $\alpha=1$ and
$\alpha=+\infty$. However we can see $D^{\alpha}_X$ as a function of
$\alpha$ for fixed distributions and consider the limits to get:
\begin{align*}
	D^{1}_X(\mu_1||\mu_2) &\defeq {\bf KL}_X(\mu_1||\mu_2) , \\
	D^{\infty}_X(\mu_1||\mu_2) &\defeq \log\sup_x \frac{\mu_1(x)}{\mu_2(x)} .
\end{align*}
The first limit is the well-known KL divergence, while the second limit is the
\emph{max divergence} that bounds the pointwise ratio of probabilities; standard
$(\varepsilon, 0)$-differential privacy bounds this divergence on distributions
from adjacent inputs.

There are several notions of differential privacy based on R\'enyi divergence,
differing in whether the bound holds for all orders $\alpha$ or just some
orders. The first notion we consider is R\'enyi Differential Privacy
(RDP)~\cite{Mironov17}.
\begin{definition}[R\'enyi Differential Privacy (RDP)~\cite{Mironov17}]
Let $\alpha \in [1,\infty)$. A randomized algorithm $\mathcal{M}:X\to\prob(Y)$ is 
\emph{$(\alpha,\rho)$-R\'enyi differentially private} if for every
  pair $x_0$ and $x_1$ of adjacent inputs, we have 
$$
D^{\alpha}_X(\mathcal{M}(x_0)||\mathcal{M}(x_1)) \leq \rho.
$$
\end{definition}
Renyi Differential privacy considers a fixed value of $\alpha$. In
contrast, zero-Concentrated Differential Privacy (zCDP)~\cite{BunS16},
a simplification of Concentrated Differential Privacy (CDP)~\cite{dwork2016concentrated},
quantifies over all possible $\alpha>1$.
\begin{definition}[zero-Concentrated Differential Privacy (zCDP)~\cite{BunS16}]
A randomized algorithm $\mathcal{M}:X\to\prob(Y)$ is 
\emph{$(\xi,\rho)$-zero concentrated differentially private} if for  every
 pairs of adjacent inputs $x_0$ and $x_1$, we have
\begin{equation}\label{eq:zCDP_definition}
  \forall{\alpha > 1}.~ D^\alpha_Y(\mathcal{M}(x_0)||\mathcal{M}(x_1)) \leq \xi + \alpha\rho.
\end{equation}
\end{definition}
Truncated Concentrated Differential Privacy
(tCDP)~\cite{BDRS18} quantifies over all $\alpha$ below a given threshold.
\begin{definition}[Truncated Concentrated Differential Privacy (tCDP)~\cite{BDRS18}]
A randomized algorithm $\mathcal{M}:X\to\prob(Y)$ is 
\emph{$(\rho,\omega)$-truncated concentrated differentially private} if for  every
 pairs of adjacent inputs $x_0$ and $x_1$,  we have
\begin{equation}\label{eq:tCDP_definition}
  \forall{1 < \alpha < \omega}.~ D^\alpha_Y(\mathcal{M}(x_0)||\mathcal{M}(x_1)) \leq \alpha\rho.
\end{equation}
\end{definition}
These notions are all motivated by bounds on the \emph{privacy loss} of a
randomized algorithm. This quantity is defined by
\[
  \mathcal{L}^{x_0 \to x_1}(y) \defeq \frac{\Pr[\mathcal{M}(x_0) = y]}{\Pr[\mathcal{M}(x_1) = y]} ,
\]
where $x_0$ and $x_1$ are two adjacent inputs. Intuitively, the privacy loss
measures how much information is revealed by an output $y$. While output values
with a large privacy loss are highly revealing---they are far more likely to
result from a private input $x_0$ rather than a different private input
$x_1$---if these outputs are only seen with small probability then it may be
reasonable to discount their influence. Each of the privacy definitions above bounds
different moments of this privacy loss, treated as a random variable when $y$ is
drawn from the output of the algorithm on input $x_0$. The following table
summarizes these bounds.
\begin{center}
  \begin{tabular}{ll}
    \toprule
    Privacy
    &
    Bound on privacy loss $\mathcal{L} = \mathcal{L}^{x_0 \to x_1}$
    \\
    \midrule
    $(\varepsilon,\delta)$-DP
    &
    $\Pr_{y \sim \mathcal{M}(x_0)} [ \mathcal{L}(y) \leq e^\varepsilon ] \geq 1 - \delta$
    \\
    $(\alpha,\rho)$-RDP
    &
    $\EE_{y \sim \mathcal{M}(x_1)} [\mathcal{L}(y)^\alpha] \leq e^{(\alpha - 1)\rho}$
    \\
    $(\xi,\rho)$-zCDP
    &
    $\forall{\alpha \in (1, \infty)}.~ \EE_{y \sim \mathcal{M}(x_1)}
      [\mathcal{L}(y)^\alpha] \leq e^{(\alpha - 1)(\xi + \alpha\rho)}$
    \\
    $(\omega,\rho)$-tCDP
    &
    $\forall{\alpha \in (1, \omega)}.~ \EE_{y \sim \mathcal{M}(x_1)} [\mathcal{L}(y)^\alpha] \leq e^{(\alpha - 1)\alpha\rho}$
    \\
    \bottomrule
  \end{tabular}
\end{center}

In particular, DP bounds the maximum value of the privacy
loss,\footnote{%
  Technically speaking, this is true only for sufficiently well-behaved distributions \cite{Meiser18}.}
$(\alpha,\cdot)$-RDP bounds the $\alpha$-moment, zCDP bounds all moments, and
$(\cdot, \omega)$-tCDP bounds the moments up to some cutoff $\omega$.  Many
conversions are known between these definitions; for instance, RDP, zCDP, and tCDP are known to sit between $(\varepsilon, 0)$ and
$(\varepsilon, \delta)$-differential privacy in terms of expressivity, up to
some modification in the parameters. While this means that RDP, zCDP, and tCDP
can sometimes be analyzed by reduction to standard differential privacy,
converting between the different notions requires weakening the parameters and
often the privacy analysis is simpler or more precise when working with RDP,
zCDP, or tCDP directly. The interested reader can refer to
the original papers~\cite{BunS16,Mironov17, BDRS18}.


\section{$k$-generated divergences}
\subsection{Background and notation}
We use standard notation and terminology from discrete probability.
%
%
For every $x\in X$, we denote by $\dirac_x$ the Dirac distribution centered at $x$
defined by $\dirac_x(x') = 1$ if $x = x'$ and $\dirac_x(x') = 0$
otherwise.
%
For any probability distribution $\mu \in \prob(X)$ and $\gamma \colon
X \to \prob(Y)$, we define $\gamma(\mu) \in\prob(Y)$ to be
$(\gamma(\mu))(y) = \sum_{x \in X} (\gamma(x))(y) \cdot \mu(x)$
for every $y \in Y$.
For any function $\gamma \colon X \to Y$,
as an abuse of notation, we define $\gamma(\mu) \in \prob(Y)$ to be
$\{x \mapsto \dirac_{\gamma(x)}\}(\mu)$, equivalently, 
$(\gamma(\mu))(y) = \sum_{x \in \gamma^{-1}(y)}\mu (x)$ for every $y \in Y$.
Every function $\gamma \colon X \to Y$ can be regarded as 
$\{x \mapsto \dirac_\gamma(x)\} \colon X \to \prob(Y)$.

\subsection{Divergences between distributions}
We start from a very general definition of divergences.  Our notation includes
the domain of definition of the divergence; this distinction will be important
when introducing the concept of $k$-generatedness.

\begin{definition}
A \emph{divergence} is a family 
$\Delta = \{\Delta_X\}_X$ of functions
\[
\Delta_X \colon \prob(X) \times \prob(X) \to [0,\infty].
\]
We use the notation $\ddist{\mu_1}{\mu_2}{X}$ to denote the divergence
between distributions $\mu_1$ and $\mu_2$ over $X$.
\end{definition}
Our notion of divergence subsumes the general notion of $f$-divergence from the
literature~\cite{csiszar63,csiszar04}.  In particular, this includes the
$\varepsilon$-divergence~\cite{BartheOlmedo2013} used to formulate
$(\varepsilon,\delta)$-differential privacy:
\[
  \Delta^\varepsilon_X(\mu_1||\mu_2) \defeq
  \sup_{S \subseteq X}(\Pr[\mu_1 \in S] -  e^\varepsilon \Pr[\mu_2 \in S]).
\]
Specifically, a randomized algorithm
$\mathcal{M} \colon X \to \prob(Y)$ is
$(\varepsilon,\delta)$-differentially private if and only if for every
pair of adjacent inputs $x_0$ and $x_1$, we have
\[
\Delta^\varepsilon_Y(\mathcal{M}(x_0)||\mathcal{M}(x_1)) \leq \delta.
\]
Many useful properties of divergences have been explored in the
literature. Our technical development will involve the following two
properties.
\begin{itemize}
\item A divergence $\Delta$ satisfies the
  \emph{data-processing inequality}
  iff for every $\gamma \colon X\to \prob(Y)$,
  $\Delta_Y(\gamma(\mu_1)||\gamma(\mu_2))\leq \Delta_X(\mu_1 || \mu_2)$.
\item A divergence $\Delta$ is \emph{quasi-convex} iff for every
  $\alpha_1,\ldots,\alpha_m\in [0,1]$ such that  $\textstyle{\sum_{m = 1}^N}
  \alpha_m = 1$ and every discrete set $X$,
\[
\Delta_X ({\textstyle \sum_{m = 1}^N} \alpha_m d_{1,m} || {\textstyle \sum_{m = 1}^N} \alpha_m d_{2,m})
\leq \max_m \Delta_X (d_{1,m} || d_{2,m}).
\]
\end{itemize}
These properties are satisfied by many common divergences. Besides
R\'enyi divergences, they also hold for all
$f$-divergences~\cite{csiszar63,csiszar04}. We will consider only
divergences satisfying them in the following. 

\subsection{$k$-cuts of divergences}
We now introduce a technical construction that will be useful in the rest of
the paper.
\begin{definition}
Let $k\in\mathbb{N}\cup\{\infty\}$.
For any divergence $\Delta = \{\Delta_X\}_{X\colon\mathrm{set}}$,
we define \emph{a $k$-cut} $\overline{\Delta}^k = \{\overline{\Delta}^k_X\}_{X\colon\mathrm{set}}$ as follows:
we fix a set $Y$ with cardinality $k$, i.e. $|Y| = k$, and define
\[
\overline{\Delta}^k_X(\mu_1||\mu_2) \defeq \sup_{\gamma \colon X \to \prob (Y)}\Delta_Y(\gamma (\mu_1)|| \gamma (\mu_2)).
\]
\end{definition}

For divergences $\Delta$ that satisfy the data-processing inequality, then the
$k$-cut is well-defined: it does not depend on the choice of $Y$. 
\begin{lemma}
If a divergence $\Delta$ satisfies the data-processing inequality, we have the
inequality $\overline{\Delta}^k \leq \Delta$ and the equality
$\overline{\Delta}^k_Y = \Delta_Y$ for any set $Y$ with $|Y| = k$.
\end{lemma}
So, without loss of generality in the sequel we will refer to this as ``the''
$k$-cut.

Another interesting property of $k$-cuts is that a $k$-cut $\overline{\Delta}^k$
of a divergence $\Delta$ satisfies the data-processing inequality, even if the
original divergence $\Delta$ does not satisfy it.

Without loss of generality, we can assume the function $\gamma$ in the definition of
a $k$-cut to be deterministic.
This can be proved by a weak version of Birkhoff-von Neumann theorem, which decomposes every probabilistic decision rule into a convex combination of deterministic ones.

\begin{theorem}[Weak Birkhoff-von Neumann] \label{thm:weak-bvn}
Let $k,l \in \mathbb{N}$ and $k > l$. Let $X$ and $Y$ such that
$|X|=k$ and $|Y|=l$. Then for any $\gamma \colon X \to \prob(Y)$,
there exist $N \in \mathbb{N}$, $\gamma_1,\ldots,\gamma_N \colon X \to Y$ and
$a_1, \ldots, a_N \in [0,1]$ such that $\sum_{m = 1}^N a_m = 1$
and $\gamma(x) = \sum_{m = 1}^N a_m \dirac_{\gamma_m(x)}$ ($x\in X$).
\end{theorem}

This fact, allow us to consider simplified formulations of a $k$-cut
of a divergence. Examples of this fact that will be useful in the
sequel are the $2$-cut and $3$-cut of the R\'enyi divergence of order
$\alpha$. These can be reformulated as follows:
\begin{align*}
\overline{D^\alpha}^{2}_X(\mu_1||\mu_2)
&=\sup_{\substack{S \subseteq X}}
\frac{1}{\alpha-1}\log\left\{
\begin{array}{l@{}}
\Pr[\mu_1 \in S]^\alpha\Pr[\mu_2 \in S]^{1-\alpha}\\
+\Pr[\mu_1 \notin S]^\alpha\Pr[\mu_2 \notin S]^{1-\alpha}
\end{array}
\right\},\\
\overline{D^\alpha}^{3}_X(\mu_1||\mu_2)
&=
\sup_{\substack{S_1,S_2 \subseteq X,\\ S_1 \cap S_2 = \emptyset}}
{\frac{1}{\alpha-1}}\log
\left\{
\begin{array}{l@{}}
\Pr[\mu_1 \in S_1]^\alpha\Pr[\mu_2 \in S_1]^{1-\alpha}\\
+\Pr[\mu_1 \in S_2]^\alpha\Pr[\mu_2 \in S_2]^{1-\alpha}\\
+\Pr[\mu_1 \notin  S_1\cup S_2]^\alpha\Pr[\mu_2 \notin  S_1\cup S_2]^{1-\alpha}
\end{array}
\right\}.
\end{align*}

\subsection{$k$-generatedness of divergences}

We now introduce the notion of {$k$-generatedness}. Informally,
$k$-generatedness is a measure of the number of decisions that are needed in an hypothesis test to characterize a divergence.
\begin{definition}
Let $k\in\mathbb{N}\cup\{\infty\}$.
A divergence $\Delta$ is
\emph{$k$-generated} if a $k$-cut $\overline{\Delta}^k$ of $\Delta$ is equal to $\Delta$ itself.  
\end{definition}
$k$-generatedness can also be reformulated as follows:
\begin{lemma}\label{lem:well-definedness_of_k-generation}
If $\Delta = \{\Delta_X\}_{X\colon\mathrm{set}}$ is $k$-generated,
for \emph{any} set $Y$ with $|Y| = k$, we have 
\[
\Delta_X(\mu_1||\mu_2) = \sup_{\gamma \colon X \to \prob (Y)}\Delta_Y(\gamma(\mu_1)||\gamma(\mu_2)).
\]
\end{lemma}

\begin{lemma}\label{fact:basic}
The following basic properties hold for all $k$-generated divergences.
\begin{itemize}
\item If $\Delta$ is $1$-generated, then $\Delta$ is constant, i.e. there exists $c \in [0,\infty]$ such that for every $X$ and every $\mu_1,\mu_2\in\prob(X)$, we have $\Delta_X (\mu_1 || \mu_2) = c$.
\item If $\Delta$ is $k$-generated, then it is also $k+1$-generated.
\item  If $\Delta$ has the data-processing inequality, then it is at least $\infty$-generated.
\item Every $k$-cut of a divergence $\Delta$ is $k$-generated.
\end{itemize}
\end{lemma}

To compare a $k$-generated divergence and a divergence, we have the
following lemma where all the inequalities are defined pointwise.

\begin{lemma}\label{fact:comparison}
Consider a divergence $\Delta$ and a $k$-generated divergence $\Delta'$. For any $k$-cut $\overline{\Delta}^k$ of $\Delta$,
\[
\Delta' \leq \Delta \implies \Delta' \leq \overline{\Delta}^k.
\]
Also, if $\Delta$ has the data-processing inequality, the $k$-cut is the greatest $k$-generated divergence below $\Delta$:
\[
\Delta' \leq \Delta \iff \Delta' \leq \overline{\Delta}^k \leq \Delta.
\]
\end{lemma}


\subsubsection{A $2$-generated divergence for DP}

The divergence $\Delta^\varepsilon$ that can be used to characterize
$(\varepsilon,\delta)$-DP is $2$-generated. This
implies that DP can be characterized completely by its
hypothesis testing interpretation. 

\begin{theorem}\label{varepsilon-divergence_is_2-generated}
The $\varepsilon$-divergence $\Delta^\varepsilon$ is $2$-generated.
\end{theorem}

Since $\Delta^\varepsilon$ is quasi-convex and satisfies data-processing
inequality, the $2$-cut can be reformulated as:
\[
\overline{\Delta^\varepsilon}^{2}_X(\mu_1||\mu_2)
=\sup_{S \subseteq X}
(\Pr[\mu_1 \in S] - e^\varepsilon \Pr[\mu_2 \in S])
\]
It is easy to show that this is exactly the same as the original
definition of $\Delta^\varepsilon$, from which follows that it is $2$-generated.

\subsubsection{R\'enyi is $\infty$-generated} \label{section:Renyi:infty}

In contrast to the divergence $\Delta^\varepsilon$,
the 2-cut of the R\'enyi divergence is not complete
with respect to the R\'enyi divergence. 

To see this let $X=\{a,b,c\}$ and let
$\mu_1,\mu_2\in\prob(X)$ be defined by
$\mu_1(a)=\mu_1(b)=\mu_1(c)=\frac{1}{3}$ and 
$\mu_2(a)=\frac{p^2}{p^2+p+1}$, $\mu_2(b)=\frac{p}{p^2+p+1}$ and 
$\mu_2(b)=\frac{1}{p^2+p+1}$.

We set $\beta > \alpha + 1$ and $p = (1/2)^{\beta/(\alpha-1)}$, a simple calculation shows:
\begin{align*}
&\overline{D^{\alpha}}^2_X(\mu_1||\mu_2) + \frac{1}{\alpha - 1} \log \frac{2^\beta + 2^{-\beta} + 1}{\max(2^{\alpha + 1},2^\beta + 1)} \leq D^{\alpha}_{X}(\mu_1||\mu_2)
\end{align*}
The difference is quantitatively small, but it is nevertheless
strictly positive. This shows also that the R\'enyi divergence is not
$2$-generated.

Similary, one can show that the 3-cut is not complete, that the 4-cut is not
complete, etc. In fact, R\'enyi divergence is exactly $\infty$-generated.
Indeed, R\'enyi divergence satisfies the data-processing inequality, hence it is
at most $\infty$-generated.  Moreover, any $f$-divergences whose weight function
$f$ is strictly convex is not $k$-generated for any finite $k$.  The formulation
of R\'enyi divergence of order $\alpha$ given by $\exp((\alpha -
1)D^{\alpha}_{X}(\mu_1||\mu_2))$ is an $f$-divergence related to the weight
function $t \mapsto t^\alpha$, which is strictly convex. Since the logarithm
function is continuous on $(0,\infty)$ and strictly monotone, we conclude that
the R\'enyi divergence is $\infty$-generated. The formal details can be found in
the appendix.

\section{Hypothesis Testing Interpretation of Divergences}

In this section, we give an hypothesis testing characterization
similar to the one that differential privacy satisfies for the $2$-cut of
an arbitrary divergence. 

We first define privacy regions for divergences using their $2$-cuts. 
\begin{definition}
For any divergence $\Delta$, we define its \emph{privacy region} $R^{\Delta}(\rho) \subseteq [0,1] \times [0,1]$ by
\[
R^{\Delta}(\rho)
\defeq
\Set{(x,y) |
\overline{\Delta}^2_{\{\mathtt{Acc},\mathtt{Rej}\}}((1-x)\mathbf{d}_{\mathtt{Acc}}+x\mathbf{d}_{\mathtt{Rej}} || y\mathbf{d}_{\mathtt{Acc}}+(1-y)\mathbf{d}_{\mathtt{Rej}}) \leq \rho
}.
\]
\end{definition}

Notice that if $\Delta$ satisfies the data-processing inequality, or is $2$-generated, then
$\overline{\Delta}^2_{\{\mathtt{Acc},\mathtt{Rej}\}}$ in the
definition above can be replaced by $\Delta_{\{\mathtt{Acc},\mathtt{Rej}\}}$.

As an example, we can give the  privacy region of DP.
\[
R^{\Delta^{\varepsilon}}(\delta)
=
\Set{(x,y) |
1-x \leq e^\varepsilon y + \delta, \quad
x \leq e^\varepsilon (1-y) + \delta
}
\]


Privacy regions are intimately related to the hypothesis testing
interpretation of privacy definitions based on divergences.

\begin{theorem}\label{2-generated_and_HT}
Let $\mu_1,\mu_2 \in \prob (X)$.
$\overline{\Delta}^2_X( \mu_1 || \mu_2) \leq \rho$ holds if and only if
for any $\gamma \colon X \to \prob (\{\mathtt{Acc},\mathtt{Rej}\})$,
\[
(\Pr[\gamma(\mu_1) = \mathtt{Rej}],\Pr[\gamma(\mu_2) = \mathtt{Acc}]) \in R^{\Delta}(\rho).
\]
%
\end{theorem}
In the theorem above, the functions $\gamma \colon X \to \prob
(\{\mathtt{Acc},\mathtt{Rej}\})$ can be seen as probabilistic decision rules.
Moreover, the privacy region can be actually relaxed to
$R^{\Delta}(\rho) \cup \Set{(x,y) | x + y \geq 1}$ since for any decision rule $\gamma \colon X \to \prob (\{\mathtt{Acc},\mathtt{Rej}\})$, we can take its negation $\neg\gamma$.
Hence we do not need to check the cases of $\Pr[\gamma(\mu_1) =
\mathtt{Rej}] + \Pr[\gamma(\mu_2) = \mathtt{Acc}] > 1$. This also
corresponds to the symmetry that we have in their graphical representations.

Finally, if a divergence $\Delta$ is quasi-convex, we also have the
equivalent of Theorem \ref{2-generated_and_HT} under deterministic decision rules.
In this case, we have the following reformulation. Let $\mu_1,\mu_2 \in \prob (X)$.
$\overline{\Delta}^2_X( \mu_1 || \mu_2) \leq \rho$ iff
for any $S \subseteq X$,
\[
(\Pr[\mu_1 \in S],\Pr[\mu_2 \notin S]) \in R^{\Delta}(\rho).
\]

This give us  the hypothesis testing characterization of DP, since  the $\varepsilon$-divergence is $2$-generated and quasi-convex.
\begin{corollary}
Let $\mu_1,\mu_2 \in \prob (X)$.
Set $\varepsilon, \delta \geq 0$.
$\Delta^{\varepsilon}_X( \mu_1 || \mu_2) \leq \delta$ iff
for any $S \subseteq X$,
\[
(\Pr[\mu_1 \in S],\Pr[\mu_2 \notin S]) \in R^{\Delta^{\varepsilon}}(\delta).
\]
\end{corollary}

We conclude this section by stressing that
Theorem~\ref{2-generated_and_HT}  tell us two important things:

\begin{itemize}
\item Every privacy definition similar to differential privacy but based on a $2$-generated divergence  is characterized completely by
  its hypothesis testing interpretation.
\item
For every  privacy definition similar to differential privacy but
based on an arbitrary divergence we can have an  hypothesis testing
interpretations by considering its $2$-cut. However, this
characterization will not be necessarily complete.
\end{itemize}
The second remark applies in particular to relaxations of differential privacy based
on the R\'enyi divergence: if we want to have the hypothesis testing
interpretation for one of these relaxations we can use the $2$-cut of the R\'enyi divergence.

\section{Applications}
In this section we will use the technical tools we developed in
the previous sections to better study the relations between different
privacy definitions.
\subsection{Conversions from Divergences to DP}

Privacy regions can be used to give better conversion rules between
privacy definitions based on divergences and differential privacy.
Let $\Delta' = \{\Delta'_X\}_{X\colon\mathrm{set}}$ be a divergence satisfying
the data-processing inequality.  We want to find the minimal parameters
$(\varepsilon(\rho),\delta(\rho))$ such that $\Delta'_X(\mu_1||\mu_2)  \leq
\rho$ implies $\Delta^{\varepsilon(\rho)}_X(\mu_1||\mu_2) \leq
\delta(\rho)$.
%

By Theorem \ref{2-generated_and_HT} and Lemma \ref{fact:comparison},
$R^{\Delta'}(\rho) \subseteq R^{\Delta^{\varepsilon(\rho)}}(\delta(\rho))$ holds if and only if
for any pair $\mu_1,\mu_2 \in \prob(X)$, 
\[
\Delta'_X(\mu_1||\mu_2) \leq \rho \implies \Delta^{\varepsilon(\rho)}_X(\mu_1||\mu_2) \leq \delta(\rho).
\]
This means that to find a good conversion law we can just compare the privacy regions.

\subsubsection{Better Conversion from RDP to DP}
\begin{figure}[h]
  \centering
  \includegraphics[height=7cm]{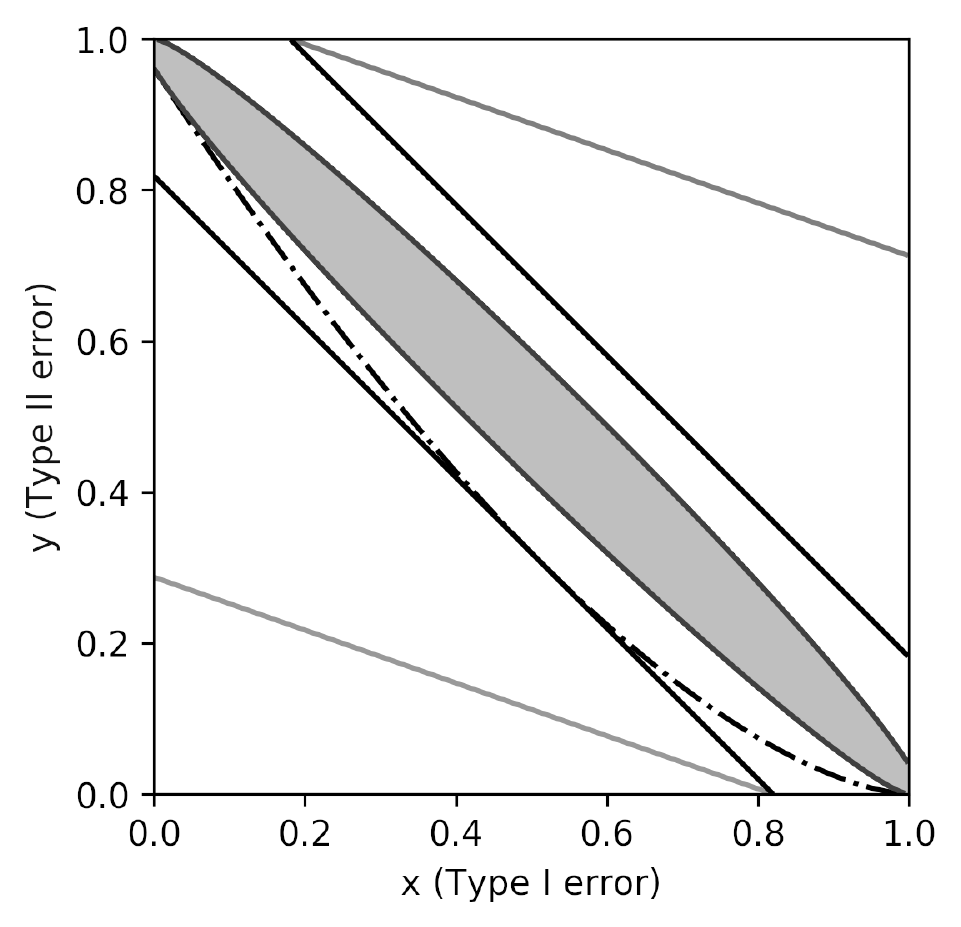}
  \caption{A refined conversion law from RDP to DP. The gray region is $R^{D^\alpha}(\rho)$. The gray and black lines show original and refiend DP-bounds for the same $\delta$.}
\end{figure}
Using privacy regions, we can refine Mironov's conversion law from RDP to DP in a simple way.
\begin{lemma}[{\cite[Prop. 3]{Mironov17}}]\label{lemma:Mironov}
If a mechanism $\mathcal{M}$ is $(\alpha,\rho)$-RDP then the mechanism is also
$(\rho - \log \delta/(\alpha-1),\delta)$-DP for any $0 < \delta < 1$.
\end{lemma}
The privacy region of R\'enyi divergence is given by
\[
R^{D^\alpha}(\rho)
=
\Set{(x,y) | x^\alpha (1-y)^{1-\alpha} + (1-x)^\alpha y^{1-\alpha} \leq e^{\rho(\alpha - 1)}}.
\]
Here we assume $0^{1-\alpha} = 0$.

By an extension of Lemma 15, to find $\varepsilon$ satisfying  
\[
\forall{X}.\forall{\mu_1,\mu_2\in\prob(X)}.~D^\alpha_X(\mu_1||\mu_2) \leq \rho \implies \Delta^\varepsilon_X(\mu_1 || \mu_2) \leq \delta,
\]
it is necessary and sufficient to find $\varepsilon$ satisfying 
\[
\forall{X}.\forall{\mu_1,\mu_2\in\prob(X)}.~\overline{D^\alpha}^2_X(\mu_1||\mu_2) \leq \rho \implies \Delta^\varepsilon_X(\mu_1 || \mu_2) \leq \delta.
\]
By Theorem 18, this is equivalent to find $\varepsilon$ satisfying 
$R^{D^\alpha}(\rho) \subseteq R^{\Delta^\varepsilon}(\delta)$.
Inspired from Mironov's proof of conversion law from RDP to DP \cite[Propisition 3]{Mironov17}: we obtain,
\begin{align*}
\lefteqn{x^\alpha (1-y)^{1-\alpha} + (1-x)^\alpha y^{1-\alpha} \leq e^{\rho(\alpha - 1)}}\\
&\implies (1-x)^\alpha y^{1-\alpha} \leq e^{\rho(\alpha - 1)} \\
&\implies (1-x) \leq (e^{\rho} y)^{\frac{\alpha-1}{\alpha}} \tag{\dag}\\
&\implies (e^{\rho} y > \delta^{\frac{\alpha}{\alpha - 1}} \implies (1-x) \leq e^{\rho - \log d/(\alpha-1)} y)\\
& \quad \land (e^{\rho} y \leq \delta^{\frac{\alpha}{\alpha - 1}} \implies (1-x) \leq \delta) \\
& \implies (1-x) \leq e^{\rho - \log d/(\alpha-1)} y + \delta \tag{\ddag}.
\end{align*}
The equality (\ddag) derives original Mironov's result \cite[Propisition 3]{Mironov17}.
Now, starting from (\dag), we have a better bound for DP as follows: consider a curve $C$ given by the equation
\[
1-x = (e^{\rho}y)^{\frac{\alpha-1}{\alpha}}
\iff
x = 1-(e^{\rho}y)^{\frac{\alpha-1}{\alpha}}
\].
We have the derivative of $x$ as follows:
\[
\frac{dx}{dy}
=
-\frac{\alpha-1}{\alpha}e^{\frac{\alpha-1}{\alpha}\rho}y^{-\frac{1}{\alpha}}
\]
We can take the tangent of the curve $C$ by
\[
x = \frac{dx}{dy}(t)(y - t) + (e^{\rho}(1-t))^{\frac{\alpha-1}{\alpha}}
\]
We will find parameters that a tangent of $C$ meets
$(1-x) = e^\varepsilon y + \delta$.
$x = -e^\varepsilon y - \delta +1$
We first solve 
\[
-e^\varepsilon
=
\frac{dx}{dy}(t)
=
-\frac{\alpha-1}{\alpha}e^{\frac{\alpha-1}{\alpha}\rho}t^{-\frac{1}{\alpha}}
\iff 
\varepsilon
=
\log(\frac{\alpha-1}{\alpha}) + \frac{\alpha-1}{\alpha}\rho -\frac{1}{\alpha}\log t.
\]
Next we solve
\[
1-\delta = -t\frac{dx}{dy}(t) + 1-(e^{\rho}t)^{\frac{\alpha-1}{\alpha}}
\iff
1-\delta =
\frac{\alpha-1}{\alpha}e^{\frac{\alpha-1}{\alpha}\rho}t^{-\frac{1}{\alpha}}t + 1-(e^{\rho}t)^{\frac{\alpha-1}{\alpha}}
\]
We then have 
\[
\delta = (e^{\rho}t)^{\frac{\alpha-1}{\alpha}} - \frac{\alpha-1}{\alpha}e^{\frac{\alpha-1}{\alpha}\rho}t^{-\frac{1}{\alpha}}t
=
\frac{1}{\alpha}(e^{\rho}t)^{\frac{\alpha-1}{\alpha}}
\iff
t = (\delta \alpha e^{-\frac{\alpha-1}{\alpha}\rho})^{\frac{\alpha}{\alpha-1}}
\]
Simple computations give the following:
\[
\varepsilon
=
\log(\frac{\alpha-1}{\alpha}) + \rho - \frac{\log\delta + \log \alpha}{\alpha-1}.
\]
By the symmetry of $R^{D^\alpha}(\rho)$ and $R^{\Delta^{\varepsilon}}(\delta)$, we have
\[
R^{D^\alpha}(\rho) \subseteq R^{\Delta^{\varepsilon}}(\delta).
\]
As we mentioned, it is equivalent to
\[
\forall{X}.\forall{\mu_1,\mu_2\in\prob(X)}.~
D^\alpha_X(\mu_1 || \mu_2) \leq \rho
\implies
\Delta^{\varepsilon}_X(\mu_1 || \mu_2) \leq \delta.
\]

Therefore, we have the following better conversion law:
\begin{theorem}
If a mechanism $\mathcal{M}$ is $(\alpha,\rho)$-RDP then it is
$(\rho + \log((\alpha-1)/\alpha) - (\log \delta + \log \alpha)/(\alpha-1),\delta)$-DP for any $0 < \delta < 1$.
\end{theorem}

As a conjecture, if we calculate tangents of the boundary 
of the privacy region $R^{D^\alpha}(\rho)$, we have 
\emph{optimal} conversion law from $(\alpha,\rho)$-RDP to DP.
The boundary of $R^{D^\alpha}(\rho)$ is given by the equation
\[
x^\alpha (1-y)^{1-\alpha} + (1-x)^\alpha y^{1-\alpha} = e^{\rho(\alpha - 1)}.
\]

\subsection{On Gaussian Differential Privacy}
Gaussian differential privacy (GDP) \cite[Def. 2.6]{DongRS19} has been
recently proposed as a privacy definition trading-off PMD and
PFA. This can be characterized by means of privacy regions. 
We have seen that privacy regions correspond to $2$-generated divergence.
Thus, a natural question is: can we characterize GDP using a $2$-generated divergence.
The answer is yes. We can characterize GDP by the following divergence:
\begin{align*}
&\Delta^{\mathrm{Gauss}}_X(\mu_1||\mu_2)\\
&=
\sup
\Set{ \delta | \begin{array}{l@{}} \exists \gamma \colon X \to \prob(\{\mathtt{Acc},\mathtt{Rej}\}).\\
\quad \Pr[\gamma(\mu_2) = 
\mathtt{Acc}] \\
\quad \geq \Phi(\inverse{\Phi}(\Pr[\gamma(\mu_1) = \mathtt{Rej}]) - \delta)
\end{array}
}.
\end{align*}
where $\Phi$ is the standard normal CDF.
The data-processing inequality of the divergence $\Delta^{\mathrm{Gauss}}$ is proved from \cite[Lem. 2.6]{DongRS19}
Hence, the privacy region is given as follows:
\[
R^{\Delta^{\mathrm{Gauss}}}(\delta)=
\Set{ (x,y) | \begin{array}{l@{}}
y \geq \Phi(\inverse{\Phi}(1-x) - \delta)\\
1-y \geq \Phi(\inverse{\Phi}(x) - \delta)
\end{array}
}.
\]
By Theorem \ref{2-generated_and_HT}, $\Delta^{\mathrm{Gauss}}$ is $2$-generated.
\subsection{Informativeness of $k$-cuts}
The concept of $k$-cut can be related to the ability that a divergence
has of distinguishing two 
distributions.
\begin{definition}
We say that a divergence $\Delta$ is $\delta$-distinguishing a pair $\mu_1, \mu_2 \in \prob(X)$ of probability distributions if
$\Delta_X (\mu_1,\mu_2) > \delta$.
\end{definition}
Now, consider a divergence $\Delta$ satisfying the data-processing inequality.
Then the $k$-cuts form a monotone increasing sequence:
$
\overline{\Delta}^1 \leq \overline{\Delta}^2 \leq \overline{\Delta}^3 \leq \cdots \leq \overline{\Delta}^{k} \leq \overline{\Delta}^{k+1} \leq \cdots.
$
Thus for any divergence with data-processing inequality, $k+1$-cut of $\Delta$ is always more informative than the $k$-cut of $\Delta$ for every $k \in \mathbb{N}$ in the following sense.
If the $k$-cut of a divergence is $\delta$-distinguishing a pair
$\mu_1, \mu_2 \in \prob(X)$ then the $k+1$ cut is
$\delta$-distinguishing them too: 
\[
\overline{\Delta}^{k+1}_X (\mu_1||\mu_2) \geq \overline{\Delta}^{k}_X (\mu_1||\mu_2) > \delta.
\]
For example, for the pair $\mu_1,\mu_2 \in \prob(\{a,b,c\})$ in the counterexample given in Section \ref{section:Renyi:infty}, we can find a value of $\delta$ such that
$\overline{D^\alpha}^{3}$ is $\delta$-distinguishing $\mu_1,\mu_2$ and
$\overline{D^\alpha}^{2}$ is \emph{not} $\delta$-distinguishing them.
\section{A characterization of $k$-generated divergences}
As we have seen, $k$-generated divergences satisfy a number of useful
properties; known divergences from the literature can be classified according to
this parameter $k$ - we have shown some examples here, more examples
are in the supplemental material. In the other direction, we give a simple condition to ensure
that a divergence is $k$-generated: the suprema of quasi-convex functions over size
$k$-partitions determine $k$-generated divergences.

\begin{theorem} \label{thm:countable-sup-implies-kgen}
Let $F \colon [0, 1]^{2k} \to [0, \infty]$ be a quasi-convex function.
Then the divergence $\Delta^F$ defined below is $k$-generated and quasi-convex.
\[
\Delta^F_X(\mu_1||\mu_2) \hspace{0.5em} \defeq\hspace{-1em}\sup_{\substack{\{ A_i \}_{i = 1}^k \\\text{partition of } X } }
\hspace{-1em} F\left(\mu_1(A_1),\cdots,\mu_1(A_k),\mu_2(A_1),\cdots,\mu_2(A_k)\right).
\]
\end{theorem}

\begin{proof}
The quasi-convexity is obvious from the quasi-convexity of $F \colon [0, 1]^{2k} \to [0, \infty]$.
We show the $k$-generatedness.
We take the $k$-cut with respect to the $k$-element set $\{1,2,\ldots,k\}$.
We may assume $X$ is countable. For any $\mu_1,\mu_2\in\prob(X)$, 
\begin{align*}
\lefteqn{\overline{\Delta^F}^k_X(\mu_1 || \mu_2)}\\
&= \sup_{\gamma \colon X \to \prob(\{1,2,\ldots,k\})} \Delta^F_{\{1,2,\ldots,k\}}(\gamma(\mu_1) || \gamma(\mu_2))\\
&= \sup_{\gamma \colon X \to \prob(\{1,2,\ldots,k\})} \sup_{\substack{\{ A_i \}_{i = 1}^k \\\text{partition of } \\\{1,2,\ldots,k\} } } F\left( \begin{array}{l@{}}(\gamma(\mu_1))(A_1),\cdots,(\gamma(\mu_1))(A_k),\\(\gamma(\mu_2))(A_1),\cdots,(\gamma(\mu_2))(A_k)\end{array} \right)\\
&= \sup_{\substack{\gamma \colon X \to \prob(\{1,2,\ldots,k\}) \\ p \colon \{1,2,\ldots,k\} \to \{1,2,\ldots,k\}}} F\left( \begin{array}{l@{}}(\gamma(\mu_1))(\inverse{p}(1)),\cdots,(\gamma(\mu_1))(\inverse{p}(k)),\\(\gamma(\mu_2))(\inverse{p}(1)),\cdots,(\gamma(\mu_2))(\inverse{p}(k))\end{array} \right)\\
&= \sup_{\substack{\gamma \colon X \to \prob(\{1,2,\ldots,k\}) \\ p \colon \{1,2,\ldots,k\} \to \{1,2,\ldots,k\}}} F\left( \begin{array}{l@{}}( ((p \bullet \gamma)(\mu_1))(1),\cdots,((p \bullet \gamma)(\mu_1))(k),\\((p \bullet \gamma)(\mu_2))(1),\cdots,((p \bullet \gamma)(\mu_2))(k)\end{array} \right)\\
&= \sup_{\gamma \colon X \to \prob(\{1,2,\ldots,k\})} F\left( \begin{array}{l@{}}( (\gamma(\mu_1))(1),\cdots,(\gamma(\mu_1))(k),\\(\gamma(\mu_2))(1),\cdots,(\gamma(\mu_2))(k)\end{array} \right).
\end{align*}
Here by weak Birkhoff-von Neumann thorem (countable version), every function $\gamma \colon X \to \prob(\{1,2,\ldots,k\})$ is decomposed into a (countable) convex combination
$\sum_{i \in I} a_i (\eta_{\{1,2,\ldots,k\}} \circ \gamma_i)$ of $\gamma_i \colon  X \to \{1,2,\ldots,k\}$.
Hence,
\begin{align*}
\lefteqn{\overline{\Delta^F}^k_X(\mu_1 || \mu_2)}\\
&= \sup_{\gamma \colon X \to \prob(\{1,2,\ldots,k\})} F\left( \begin{array}{l@{}}( (\gamma(\mu_1))(1),\cdots,(\gamma(\mu_1))(k),\\(\gamma(\mu_2))(1),\cdots,(\gamma(\mu_2))(k)\end{array} \right) \tag{\dag}\\
&= \sup_{\gamma \colon X \to \prob(\{1,2,\ldots,k\})}
F\left( \begin{array}{l@{}}( (\sum_{i \in I} a_i (\eta_{\{1,2,\ldots,k\}} \circ \gamma_i)(\mu_1))(1),\\\qquad \cdots,(\sum_{i \in I} a_i (\eta_{\{1,2,\ldots,k\}} \circ \gamma_i)(\mu_1))(k),\\(\sum_{i \in I} a_i (\eta_{\{1,2,\ldots,k\}} \circ \gamma_i)(\mu_2))(1),\\\qquad \cdots,(\sum_{i \in I} a_i (\eta_{\{1,2,\ldots,k\}} \circ \gamma_i)(\mu_2))(k)\end{array} \right)\\
&= \sup_{\gamma \colon X \to \prob(\{1,2,\ldots,k\})}
F\left( \begin{array}{l@{}}( \sum_{i \in I} a_i (\gamma_i(\mu_1))(1),\cdots, \sum_{i \in I} a_i (\gamma_i(\mu_1))(k),\\\sum_{i \in I} a_i (\gamma_i(\mu_2))(1),\cdots, \sum_{i \in I} a_i (\gamma_i(\mu_2))(k)\end{array} \right)\\
&\leq 
\sup_{\gamma \colon X \to \prob(\{1,2,\ldots,k\})}
\sup_{i \in I}
F\left( \begin{array}{l@{}}((\gamma_i(\mu_1))(1),\cdots,(\gamma_i(\mu_1))(k),\\
(\gamma_i(\mu_2))(1),\cdots,(\gamma_i(\mu_2))(k)\end{array} \right)\\
&\leq 
\sup_{\gamma \colon X \to \{1,2,\ldots,k\}}
F\left( \begin{array}{l@{}}((\gamma(\mu_1))(1),\cdots,(\gamma(\mu_1))(k),\\
(\gamma(\mu_2))(1),\cdots,(\gamma(\mu_2))(k)\end{array} \right)\\
&=
\sup_{\gamma \colon X \to \{1,2,\ldots,k\}}
F\left( \begin{array}{l@{}}(\mu_1(\inverse{\gamma}(1)),\cdots,\mu_1(\inverse{\gamma}(k)),\\
\mu_2(\inverse{\gamma}(1)),\cdots,\mu_2(\inverse{\gamma}(k))\end{array} \right)\\
&\leq
\sup_{\substack{\{ A_i \}_{i = 1}^k \\\text{partition of } \\\{1,2,\ldots,k\} } } F\left( \begin{array}{l@{}}(\gamma(\mu_1))(A_1),\cdots,(\gamma(\mu_1))(A_k),\\(\gamma(\mu_2))(A_1),\cdots,(\gamma(\mu_2))(A_k)\end{array} \right)\\
&= \Delta^F_X(\mu_1 || \mu_2).
\end{align*}
We have $\overline{\Delta^F}^k_X(\mu_1 || \mu_2) \leq \Delta^F_X(\mu_1 || \mu_2)$.
Conversely, by equality $(\dag)$, we also have
$\overline{\Delta^F}^k_X(\mu_1 || \mu_2) \geq \Delta^F_X(\mu_1 || \mu_2)$.
This completes the proof.
\end{proof}

This result characterizes $k$-generated quasi-convex divergences.
It also serves as a useful tool to construct new divergences with
a hypothesis testing interpretation, by varying the quasi-convex function $F$.

\section{Conclusion}
In this paper we have developed analytical tools to study the
hypothesis testing interpretation of privacy definitions similar to
differential privacy but measured with another statistical divergence. We introduced
the notions of $k$-cut and $k$-generatedness for divergences. These
notions quantifies the number of decisions that are needed in an experiment similar to
the ones used in hypothesis testing to fully characterize the
divergence. We used these notions to study the hypothesis testing
interpretation of relaxations of differential privacy based on the R\'enyi divergence.
These notions give a measure of the complexity that
tools for formal verification may have. We leave the study of this
connection for future work.

\newpage
\appendix

\section{Weak version of Birkhoff-von Neumann Theorem}

\begin{theorem}[Weak Birkhoff-von Neumann theorem (Theorem 11)]
Let $k,l \in \mathbb{N}$ and $k > l$.
For any $\gamma \colon k \to \prob(l)$,
there are $\gamma_1,\gamma_2,\ldots,\gamma_N \colon k \to l$ and $0 \leq a_1,a_2,\ldots,a_N \leq 1$ such that $\sum_{m = 1}^N a_m = 1$ and $\gamma(i) = \sum_{m = 1}^N a_m \mathbf{d}_{\gamma_m(i)}$ for any $1 \leq i \leq k$.
\end{theorem}

The cardinal $k$ can be relaxed to countable infinite cardinal $\omega$, and then the families $\{\gamma_j\}_j$ and $\{a_j\}_j$ may be countable infinite.

\begin{proof}
Consider the following matrix representation $f$ of $\gamma$:
\[
f =
\left(
\begin{array}{c@{\hspace{1em}}c@{\hspace{1em}}c}
f_{1,1} & \cdots & f_{l,1}\\
\vdots  &        & \vdots\\
f_{1,k} & \cdots & f_{l,k}
\end{array}
\right).
\]
where $f_{i,j}=\gamma(i)(j)$ and $\sum_{j=1}^N f_{i,j} = 1$ for any $1 \leq i \leq l$.

For any $h \colon k \to l$, the matrix representation $g$ of $(\{x \mapsto \mathbf{d}_x \}\circ h)$ is
\[
g =
\left(
\begin{array}{c@{\hspace{1em}}c@{\hspace{1em}}c}
g_{1,1} & \cdots & g_{l,1}\\
\vdots  &        & \vdots\\
g_{1,k} & \cdots & g_{l,k}
\end{array}
\right)
\]
satisfying that for any $1 \leq i \leq l$, there is exactly $1 \leq j \leq k$ such that $g_{i,j} = 1$ and $g_{i,s} = 0$ for $s \neq j$.
Conversely, any matrix $g$ satisfying this condition corresponds to some function $h \colon k \to l$. 
Consider the family $G$ of matrix representations of maps of the form $(\{x \mapsto \mathbf{d}_x \}\circ h)$.
We give an algorithm decomposing $f$ to a convex sum of $g$:
\begin{enumerate}
\item 
Let $r_0 = 1$ and $\tilde f_0 = f$. We have $\sum_{j} (\tilde f_0)_{i,j} = r_0$ for all $1 \leq i \leq l$.
\item
For given $0 \leq r_m \leq 1$ and $\tilde f_m$ satisfying $\sum_{j} (\tilde f_m)_{i,j} = r_m$ for all $1 \leq i \leq l$,
we define $g_{m+1}\in G$, $\alpha_{m+1} \in [0,1]$, $\tilde f_{m+1}$ and $r_{m+1} \in [0,1]$ as follows:
\begin{mathpar}
\alpha_{m+1} =
\min_s \max_t (\tilde f_m)_{s,t},

r_{m+1} = r_m - \alpha_{m+1},

(g_{m+1})_{i,j} =
\begin{cases}
1 & j = \argmax_s (\tilde f_m)_{i,s}\\
0 & \text{ (otherwise) }
\end{cases},

\tilde f_{m+1} =
\tilde f_{m} - \alpha_{m+1} \cdot g_{m+1}.
\end{mathpar}
\item If $r_{s+1}= 0$ then we terminate. Otherwise, we repeat the previous step.
\end{enumerate}
In each step, we obtain the following conditions:
\begin{itemize} 
\item We have $g_{m+1} \in G$ because $g_{m+1}$ can be written as $g_{m+1} = \{x \mapsto \mathbf{d}_x \} \circ (\lambda i.~ \argmax_s (\tilde f_m)_{i,s})$.

\item We have $0 < \alpha_{m+1}$ whenever $0 < r_{m}$ because
\[
\alpha_{m+1} = 0
\iff
\exists i.~\max_{j} (\tilde f_m)_{i,j} = 0
\implies \exists i.~r_m = \sum_{j} (\tilde f_m)_{i,j} = 0.
\]
\item We have $0 \leq (\tilde f_{m+1})_{i,j} \leq 1$ for any $(i,j)$ from the following equation:
\[
(\tilde f_{m+1})_{i,j}
=
\begin{cases}
(\tilde f_{m})_{i,j} - \min_t \max_s (\tilde f_m)_{t,s} & \text{ if } j = \argmax_s (\tilde f_m)_{i,s}\\
(\tilde f_{m})_{i,j} & \text{ otherwise }.
\end{cases}
\]
When $i = \argmin_s \max_t (\tilde f_m)_{s,t}$ and $ j = \argmax_s (\tilde f_m)_{i,s}$, we obtain $(\tilde f_{m+1})_{i,j} = 0$ while $0 < (\tilde f_{m+1})_{i,j}$.
This implies that the number of $0$ in $\tilde f_m$ increases in this operation.
\item We also have $\sum_{j} (\tilde f_{m+1})_{i,j} = r_{m+1}$ for all $1 \leq i \leq k$ because
\[
\sum_{j} (\tilde f_{m+1})_{i,j} =
\sum_{j} (\tilde f_{m})_{i,j} - \alpha_{m+1} \cdot \sum_{j} (\tilde g_{m+1})_{i,j}
=
r_{m} - \alpha_{m+1} \cdot 1 = r_{m+1}. 
\]
\end{itemize}
Therefore the construction of $g_{l}\in G$, $\alpha_{l} \in [0,1]$, $\tilde f_{l}$ and $r_{l} \in [0,1]$ terminates within $k \cdot l$ steps.
When the construction terminates at the step $N$ ($r_N = 0$ also holds), we have a convex decomposition of $f$ by $f = \sum_{m = 1}^{N} \alpha_m \cdot g_m$ where $\sum_{m = 1}^{N}\alpha_m = 1$. 
This implies 
By taking $\gamma_1,\gamma_2,\ldots,\gamma_N \colon k \to l$ such that $g_m$ is a matrix representation of $(\{x \mapsto \mathbf{d}_x \}\circ \gamma_m)$,
we obtain $\gamma(i) = \sum_{m = 1}^N a_m \mathbf{d}_{\gamma_m(i)}$ for any $1 \leq i \leq k$ with $0 \leq a_1,a_2,\ldots,a_N \leq 1$ and $\sum_{m = 1}^N a_m = 1$.
\end{proof}


\section{Omitted Proofs}

\subsection{Compositions of probabilistic processes}

For simplicity, we introduce the composition operator of probabilistic processes (inspired from \cite{Giry1982}).
For any $\gamma_1 \colon X \to \prob(Z)$ and $\gamma \colon Z \to \prob(Y)$, we define their composition $(\gamma \bullet \gamma_1) \colon X \to \prob(Z)$ by $(\gamma \bullet \gamma_1)(x) \defeq \gamma (\gamma_1(x))$.
It is easy to check that the composition $(\gamma \bullet \gamma_1)$ satisfies $(\gamma \bullet \gamma_1)(\mu) = \gamma(\gamma_1 (\mu))$ for every $\mu \in \prob(X)$.
\begin{itemize}
\item The composititon operator $\bullet$ is associative: $\gamma \bullet (\gamma_1 \bullet \gamma_2) = (\gamma \bullet \gamma_1) \bullet \gamma_2$ holds for all $\gamma_2  \colon W \to \prob(X)$, $\gamma_1  \colon X \to \prob(Z)$, and $\gamma \colon Z \to \prob(Y)$.
\item The function $\eta_X  \colon X \to \prob(X)$ defined by $\eta_X = \{ x \mapsto \mathbf{d}_X\}$ is the unit of  operator $\bullet$: we have $\gamma \bullet \eta_X = \gamma$ and $\eta_Y \bullet \gamma = \gamma$ for all $\gamma \colon X \to \prob(Y)$
\end{itemize}
Thanks to the unit law and associativity of $\bullet$ as an abuse of notations, we define 
\begin{itemize}
\item $(\gamma \bullet \gamma_1) \colon X \to \prob(Z)$ for $\gamma_1 \colon X \to Z$ and $\gamma \colon Z \to \prob(Y)$ by $\gamma \bullet (\eta_Z \circ \gamma_1)$.
\item $(\gamma \bullet \gamma_1) \colon X \to \prob(Z)$ for $\gamma_1  \colon X \to \prob(Z)$ and $\gamma \colon Z \to Y$ by $(\eta_Y \circ \gamma) \bullet \gamma_1$.
\item $(\gamma \bullet \gamma_1) \colon X \to \prob(Z)$ for $\gamma_1  \colon X \to Z$ and $\gamma \colon Z \to Y$ by $(\eta_Y \circ \gamma) \bullet (\eta_Z \circ \gamma_1)$, which is equal to $\eta_Y \circ (\gamma \circ \gamma_1)$.
\end{itemize}
Notice that, $\gamma(\mu) \in \prob(Y)$ defined under $\gamma \colon X \to Y$ and $\mu \in \prob(X)$ is exactly $(\eta_Y \circ \gamma)(\mu)$.

\subsection{Proof of the data-processing inequality of $k$-cuts}

\begin{lemma}
For any divergence $\Delta$, every $k$-cut $\overline{\Delta}^k$ satisfies data-processing inequality.
\end{lemma}
\begin{proof}
We consider the $k$-cut of $\Delta$ with respect to a set $Y$ satisfying $|Y| = k$
\[
\overline{\Delta}^k_X(\mu_1||\mu_2) \defeq \sup_{\gamma \colon X \to \prob (Y)}\Delta_Y(\gamma (\mu_1)|| \gamma (\mu_2)).
\] 
For every pair $\mu_1,\mu_2 \in \prob(X)$, and any function $\gamma_1  \colon X \to \prob(Z)$, we obtain the data-processing inequality
\begin{align*}
\overline{\Delta}^k_Z(\gamma_1 (\mu_1)||\gamma_1 (\mu_2)) 
& = \sup_{\gamma \colon Z \to \prob (Y)}\Delta_Y(\gamma(\gamma_1 (\mu_1))|| \gamma(\gamma_1 (\mu_2)))\\
& = \sup_{\gamma \colon Z \to \prob (Y)}\Delta_Y((\gamma \bullet \gamma_1)(\mu_1)|| (\gamma \bullet \gamma_1)(\mu_2))\\
& \leq \sup_{\gamma' \colon X \to \prob (Y)}\Delta_Y(\gamma' (\mu_1)|| \gamma' (\mu_2))
 = \overline{\Delta}^k_X (\mu_1 || \mu_2).
\end{align*}
The inequality is obtained by the inclusion
\[
\Set{ (\gamma \bullet \gamma_1) \colon X \to \prob Y| \gamma \colon Z \to \prob (Y) }
\subseteq \Set{ \gamma' \colon X \to \prob (Y) }.
\]
\end{proof}

\subsection{Proof of Lemma 10}

\begin{lemma}[Lemma 10]
If a divergence $\Delta$ has the data-processing inequality,
we have the inequality $\overline{\Delta}^k \leq \Delta$ and
the equality $\overline{\Delta}^k_Y = \Delta_Y$ for any set $Y$ with $|Y| = k$.
\end{lemma}
\begin{proof}
We consider the $k$-cut of $\Delta$ with respect to a set $W$ satisfying $|W| = k$
\[
\overline{\Delta}^k_X(\mu_1||\mu_2) \defeq \sup_{\gamma \colon X \to \prob (W)}\Delta_W(\gamma (\mu_1)|| \gamma (\mu_2)).
\] 
Thanks to the data-processing inequality of $\Delta$, we have $\overline{\Delta}^k \leq \Delta$: for every pair $\mu_1,\mu_2 \in \prob(X)$, we obtain
\[
\overline{\Delta}^k_X(\mu_1||\mu_2) 
= \sup_{\gamma \colon X \to \prob (W)}\Delta_W(\gamma (\mu_1)|| \gamma (\mu_2))
\leq \Delta_X(\mu_1,\mu_2).
\]

Now, we consider a set $Y$ with $|Y| = k$.
We already have $\overline{\Delta}^k_Y \leq \Delta_Y$. We want to prove $\Delta_Y \leq \overline{\Delta}^k_Y$ 
Since $|Y| = |W| = k$, there is a bijection $f \colon Y \to W$.
We then obtain for every pair $\nu_1,\nu_2 \in \prob(Y)$,
\[
\Delta_Y(\nu_1 || \nu_2)
= \Delta_Y(\inverse{f}(f(\nu_1)) || \inverse{f}(f(\nu_2)))
\leq \Delta_W( f(\nu_1) || f(\nu_2))
\leq \overline{\Delta}^k_Y(\nu_1||\nu_2)
\]
The first and second inequalities are obtained by the dataprocessing inequality and the definition of $k$-cut respectively.
\end{proof}

\subsection{Proof of Lemma 13}

\begin{lemma}[Lemma 13]
If $\Delta = \{\Delta_X\}_{X\colon\mathrm{set}}$ is $k$-generated,
for \emph{any} set $|Y|$ with $|Y| = k$, we have 
\[
\Delta_X(\mu_1||\mu_2) = \sup_{\gamma \colon X \to \prob (Y)}\Delta_Y(\gamma(\mu_1)||\gamma(\mu_2)).
\]
\end{lemma}
\begin{proof}
Suppose that $\Delta$ is equal to the $k$-cut of $\Delta$ with respect to a set $W$ satisfying $|W| = k$.
\[
\Delta_X(\mu_1||\mu_2)
= \overline{\Delta}^k_X(\mu_1||\mu_2) = \sup_{\gamma \colon X \to \prob (W)}\Delta_W(\gamma (\mu_1)|| \gamma(\mu_2)).
\]
Since $\overline{\Delta}^k$ always satisfies data-processing inequality, the divergence $\Delta$ itself do so.
We fix an arbitrary set $|Y|$ with $|Y| = k$.
Since $|Y| = |W| = k$, there is a bijection $f \colon Y \to W$.
For every pair $\mu_1,\mu_2 \in \prob(X)$, we obtain
\begin{align*}
\Delta_X(\mu_1||\mu_2)
& = \sup_{\gamma \colon X \to \prob (W)}\Delta_W(\gamma (\mu_1)|| \gamma(\mu_2))\\
& = \sup_{\gamma \colon X \to \prob (W)}\Delta_W(f(\inverse{f}(\gamma (\mu_1)))||f(\inverse{f}(\gamma (\mu_1))))\\
& \leq \sup_{\gamma \colon X \to \prob (W)}\Delta_Y(\inverse{f}(\gamma (\mu_1))||\inverse{f}(\gamma (\mu_1)))\\
& = \sup_{\gamma \colon X \to \prob (W)}\Delta_Y((\inverse{f} \bullet \gamma) (\mu_1)||(\inverse{f} \bullet \gamma)(\mu_1))\\
& \leq \sup_{\gamma \colon X \to \prob (Y)}\Delta_Y(\gamma(\mu_1)||\gamma(\mu_2)) \leq \Delta_X(\mu_1||\mu_2).
\end{align*}
Here, the first and last inequalities are obtained from the data-processing inequality of $\Delta$.
The second inequality is proved from the inclusion
\[
\Set{ (\inverse{f} \bullet \gamma) \colon X \to \prob(Y) | \gamma \colon X \to \prob (W)}
\subseteq \Set{\gamma \colon X \to \prob (Y)}.
\]
\end{proof}

\subsection{Proof of Basic Properties of $k$-generatedness (Lemma 14)}

\begin{lemma}[Lemma 14 (1)]
If $\Delta$ is $1$-generated, then $\Delta$ is constant, i.e. there exists $c \in [0,\infty]$ such that for every $X$ and every $\mu_1,\mu_2\in\prob(X)$, we have $\Delta_X (\mu_1 || \mu_2) = c$.
\end{lemma}
\begin{proof}
When $\Delta$ is $1$-generated, there is a singleton set $\{a\}$ such that, for every pair $\mu_1,\mu_2 \in \prob(X)$,
\[
\Delta_X (\mu_1 || \mu_2) = \sup_{\gamma \colon X \to \prob(\{a\})}\Delta_{\{a\}}(\gamma(\mu_1) || \gamma(\mu_2)). 
\]
Now, the set $\prob(\{a\})$ is a singleton set $\{ \mathbf{d}_{a} \}$, and therefore both $\gamma(\mu_1)$ and $\gamma(\mu_1)$ are equal to $ \mathbf{d}_{a}$ for every $\gamma \colon X \to \prob(\{a\})$ and every pair $\mu_1,\mu_2 \in \prob(X)$.
Hence, $\Delta_X (\mu_1 || \mu_2) = c$ where $c = \Delta_{\{a\}}(\mathbf{d}_{a} || \mathbf{d}_{a})$.
\end{proof}

\begin{lemma}[Lemma 14 (2)]
If $\Delta$ is $k$-generated, then it is also $k+1$-generated.
\end{lemma}
\begin{proof}
Suppose that $\Delta$ is equal to the $k$-cut of $\Delta$ with respect to a set $W$ satisfying $|W| = k$.
\[
\Delta_X(\mu_1||\mu_2)
= \overline{\Delta}^k_X(\mu_1||\mu_2) = \sup_{\gamma \colon X \to \prob (W)}\Delta_W(\gamma (\mu_1)|| \gamma(\mu_2)).
\]
Let $V$ be an arbitrary set with $|V| = k+1$.
We define the $k+1$-cut of $\overline{\Delta}^k$ with respect to the set $V$.
\[
\overline{\overline{\Delta}^k}^{k+1}_X(\mu_1||\mu_2) \defeq
\sup_{\gamma \colon X \to \prob (V)} \overline{\Delta}^k_V(\gamma (\mu_1)|| \gamma(\mu_2)).
\]
We then have 
\begin{align*}
\overline{\overline{\Delta}^k}^{k+1}_X(\mu_1||\mu_2)
&=
\sup_{\gamma \colon X \to \prob (V)} \overline{\Delta}^k_V(\gamma (\mu_1)|| \gamma(\mu_2))\\
&=
\sup_{\gamma \colon X \to \prob (V)}\sup_{\gamma_1 \colon V \to \prob (W)} \Delta_W(\gamma_1 (\gamma(\mu_1))|| \gamma_1 (\gamma(\mu_2)))\\
&=
\sup_{\gamma \colon X \to \prob (V)}\sup_{\gamma_1 \colon V \to \prob (W)} \Delta_W((\gamma_1 \bullet \gamma)(\mu_1)|| (\gamma_1 \bullet \gamma)(\mu_2))\\
&\stareq
\sup_{\gamma' \colon X \to \prob (W)} \Delta_W(\gamma'(\mu_1)|| \gamma'(\mu_2)) = \overline{\Delta}^k_X(\mu_1||\mu_2)\\
\end{align*}
The equality is obtained by the equality
\[
\Set{ (\gamma_1 \bullet \gamma) \colon X \to \prob (W) |\begin{array}{l@{}}\gamma \colon X \to \prob (V),\\\gamma_1 \colon V \to \prob (W)\end{array}}
=
\Set{\gamma' \colon X \to \prob (W)}
\]
The inclusion $\subseteq$ is obvious.
We show the reverse inclusion $\supseteq$.
Since $|V| \geq |W|$, there is a pair of function $f \colon W \to V$ and $g \colon V \to W$ such that $g \circ f = \mathrm{id}_W$.
Then, every $\gamma' \colon X \to \prob (W)$ can be decomposed into
$\gamma' = \gamma_1 \bullet \gamma$ where $\gamma =  (f \bullet \gamma)$ and
$\gamma_1 = g$ (strictly, $\gamma =  ((\eta_V \circ f) \bullet \gamma)$ and $\gamma_1 = \eta_W \circ g$ ).
\end{proof}

\begin{lemma}[Lemma 14 (3)]
If $\Delta$ has the data-processing inequality, then it is at least $\infty$-generated.
\end{lemma}
\begin{proof}
We fix a pair $\mu_1,\mu_2 \in \prob(X)$.
The set $Y = \supp(\mu_1)\cup\supp(\mu_1)$ is at most countable.
Hence there are two functions $f \colon X \to \mathbb{N}$ and $g \colon \mathbb{N} \to X$ such that $(g \circ f)(x) = x$ for every $x \in Y$.
We then have $\mu_1 = (g\circ f)(\mu_1)$ and $\mu_2 = (g\circ f)(\mu_2)$.
Thus,
\begin{align*}
\Delta_{X}(\mu_1||\mu_2)
&= \Delta_{X}((g\circ f)(\mu_1)||(g\circ f)(\mu_2))\\
&\leq \Delta_{\mathbb{N}}( f(\mu_1)|| f(\mu_2))\\
&\leq \sup_{\gamma \colon X \to \prob(\mathbb{N})} \Delta_{\mathbb{N}}( \gamma(\mu_1)|| \gamma(\mu_2))\\
&\leq \Delta_{X}(\mu_1||\mu_2).
\end{align*}
The last part is an $\infty$-cut.
The first and last inequality is obtained by the data-processing inequality.
The second one is obvious ($f \colon X \to \mathbb{N}$ is regarded as $\{x \mapsto \mathbf{d}_{f(x)}\} \colon X \to \prob(\mathbb{N})$). 
\end{proof}

\begin{lemma}[Lemma 14 (4)]
Every $k$-cut of a divergence $\Delta$ is always $k$-generated.
\end{lemma}
\begin{proof}
We can prove $\overline{\overline{\Delta}^k}^{k} = \overline{\Delta}^k$ in a almost the same way as Lemma 14 (2).
\end{proof}

\paragraph{Continuity of divergence (Lemma 14(3) in general setting)}
We can extend the results on divergences in the discrete setting to general measurable setting
using the continuity of divergences.
We say that a divergence $\Delta$ is continuous if for any pair $\mu_1, \mu_2 \in \prob(X)$,
\[
\Delta_X(\mu_1 || \mu_2)
=
\sup_{n \in \mathbb{N}} \sup_{\gamma \colon X \to \{0,1,2,\ldots,n\}}
\Delta_{\{0,1,2,\ldots,n\}} (\gamma(\mu_1) || \gamma(\mu_2)).
\]
If $\Delta$ is continuous and satisfies data-processing inequality we have $\infty$-generatedness (moreover we show  the ``countable''-generatedness) as follows:
\begin{align*}
\lefteqn{\Delta_X(\mu_1 || \mu_2)}\\
&= \sup_{n \in \mathbb{N}} \sup_{\gamma \colon X \to \{0,1,2,\ldots,n-1\}}
\Delta_{\{0,1,2,\ldots,n-1\}} (\gamma(\mu_1) || \gamma(\mu_2))\\
&= \sup_{n \in \mathbb{N}} \sup_{\gamma \colon X \to \{0,1,2,\ldots,n-1\}}
\Delta_{\{0,1,2,\ldots,n-1\}} ((g_n \circ f_n)(\gamma(\mu_1)) || (g_n \circ f_n)(\gamma(\mu_2)))\\
&= \sup_{n \in \mathbb{N}} \sup_{\gamma \colon X \to \{0,1,2,\ldots,n\}}
\Delta_{\{0,1,2,\ldots,n-1\}} (g_n((f_n \bullet \gamma)(\mu_1)) || g_n((f_n \bullet \gamma)(\mu_1)) )\\
&\leq \sup_{n \in \mathbb{N}} \sup_{\gamma \colon X \to \{0,1,2,\ldots,n-1\}}
\Delta_{\mathbb{N}} ((f_n \bullet \gamma)(\mu_1) || (f_n \bullet \gamma)(\mu_1) )\\
&\leq \sup_{\gamma \colon X \to \mathbb{N}} \Delta_{\mathbb{N}} (\gamma(\mu_1) || \gamma(\mu_2) ) \leq \overline{\Delta}^\infty_X (\mu_1,\mu_2) \\
&\leq \Delta_X(\mu_1,\mu_2).
\end{align*}
Here $f_n \colon \{0,1,2,\ldots,n-1\} \to \mathbb{N}$ is the inclusion mapping,
and $g_n \colon \mathbb{N} \to \{0,1,2,\ldots,n-1\}$ is defined by $g_n(k) = k$ if ($k < n$) and $g_n(k) = n-1$ otherwise. We have $(g_n \circ f_n) = \mathrm{id}_{\{0,1,2,\ldots,n-1\}}$.

The first and last inequalities are obtained from data-processing inequality.
The second inequality is obvious.

\subsection{Proof of Lemma 15}

\begin{lemma}[Lemma 15]
Consider a divergence $\Delta$ and a $k$-generated divergence $\Delta'$.
For any $k$-cut $\overline{\Delta}^k$ of $\Delta$,
\[
\Delta' \leq \Delta \implies \Delta' \leq \overline{\Delta}^k.
\]
Also, if $\Delta$ has the data-processing inequality, the $k$-cut is the greatest $k$-generated divergence below $\Delta$:
\[
\Delta' \leq \Delta \iff \Delta' \leq \overline{\Delta}^k \leq \Delta.
\]
\end{lemma}
\begin{proof}
Since $\Delta'$ is $k$-generated, for any choice of $Y$ with $|Y| = k$, we have 
\[
\Delta' \leq \Delta \implies \Delta'_Y \leq \Delta_Y \implies \overline{\Delta'}^k \leq \overline{\Delta}^k \iff \Delta' \leq \overline{\Delta}^k.
\]
The second statement is proved as follows:
From the first statement of this lemma and Lemma 3 (Lemma 10 in the paper), 
We have
\[
\Delta' \leq \Delta \implies \Delta' \leq \overline{\Delta}^k \leq \Delta
\]
The converse direction is obvious.
\end{proof}

\paragraph{An extended version.}

We can extend this theorem to more suitable for conversion laws of differential privacy.
\begin{lemma}[Lemma 15, extended]
Consider a divergence $\Delta$ satisfying data-processing inequality
and a $k$-generated divergence $\Delta'$.
\begin{align*}
&\forall{X}.\forall{\mu_1,\mu_2 \in \prob(X)}.(\Delta_X(\mu_1||\mu_2) \leq \delta \implies \Delta'_X(\mu_1||\mu_2) \leq \rho)\\
&\iff
\forall{X}.\forall{\mu_1,\mu_2 \in \prob(X)}.(\overline{\Delta}^k_X(\mu_1||\mu_2) \leq \delta \implies \Delta'_X(\mu_1||\mu_2) \leq \rho)
\end{align*}
\end{lemma}
\begin{proof}
($\impliedby$) Obvious from Lemma 3 (Lemma 10 in the paper).
($\implies$)
From the assumption, we obtain
\begin{align*}
&\forall{X}.\forall{\mu_1,\mu_2 \in \prob(X)}.
\forall{\gamma \colon X \to \prob(Y)}.\\
&\qquad\qquad\Delta_Y(\gamma(\mu_1)||\gamma(\mu_2)) \leq \delta \implies \Delta'_Y(\gamma(\mu_1)||\gamma(\mu_2)) \leq \rho.
\end{align*}
This implies 
\[
\forall{X}.\forall{\mu_1,\mu_2 \in \prob(X)}.(
\overline{\Delta}^k_X(\mu_1||\mu_2 \leq \delta
\implies
\overline{\Delta'}^k_X(\mu_1||\mu_2) \leq \rho
)
\]
Thanks to the $k$-generatedness of $\Delta'$, we conclude the statement of this lemma.
\end{proof}

\subsection{Proof of $2$-generatedness of $\varepsilon$-divergence}
\begin{theorem}
The $\varepsilon$-divergence $\Delta^\varepsilon$ is $2$-generated for all $\varepsilon$.
\end{theorem}
\begin{proof}
We recall that the $\varepsilon$-divergence $\Delta^\varepsilon$ is quasi-convex (moreover, jointly convex) and satisfies data-processing inequality.
We choose a set $Y = \{\mathtt{Acc},\mathtt{Rej}\}$, and take the $2$-cut of $\Delta^\varepsilon$ by
\[
\overline{\Delta^\varepsilon}^{2}_X(\mu_1||\mu_2)
=\sup_{\gamma \colon X \to \prob(\{\mathtt{Acc},\mathtt{Rej}\})} \Delta^\varepsilon_{Y} (\gamma(\mu_1) || \gamma(\mu_2))
\]
We show this is equal to the original $\Delta^\varepsilon_X(\mu_1||\mu_2)$.
Without loss of generality we may assume $X$ is at most countable.
If $X$ is an arbitrary set, we can restrict it to countable set in a similar way as
the proof of Lemma 7 (Lemma 14(3) in the paper).

By the weak Birkhoff-von Neumann Theorem (Theorem 1 in the appendix),
each $\gamma \colon X \to \prob(\{\mathtt{Acc},\mathtt{Rej}\})$ can be decomposed into 
a convex combination
$\gamma(x) = \sum_{i \in I} \alpha_i \mathbf{d}_{\gamma_i(x)}$
of functions $\gamma_i \colon X \to \{\mathtt{Acc},\mathtt{Rej}\}$ ($i \in I$) where $I$ is a countable set and
$\sum_{i \in I} \alpha_i = 1$.
By combining this and quasi-convexity and data-processing inequality of $\Delta^\varepsilon$, we obtain 
\begin{align*}
\Delta^\varepsilon (\gamma(\mu_1) || \gamma(\mu_2))
&= \Delta^\varepsilon ({\textstyle\sum_{i \in I} \alpha_i \gamma_i(\mu_1)} || {\textstyle\sum_{i \in I} \alpha_i \gamma_i(\mu_1)})\\
&= \sup_{i \in I} \Delta^\varepsilon (\gamma_i(\mu_1) || \gamma_i(\mu_1))\\
&\leq  \sup_{\gamma \colon X \to  \{\mathtt{Acc},\mathtt{Rej}\} } \Delta^\varepsilon_X( \gamma(\mu_1) || \gamma(\mu_2)).
\end{align*}
This implies
\begin{align*}
\overline{\Delta^\varepsilon}^{2}_X(\mu_1||\mu_2)
&=
\sup_{\gamma \colon X \to \prob(\{\mathtt{Acc},\mathtt{Rej}\})} \Delta^\varepsilon (\gamma(\mu_1) || \gamma(\mu_2))\\
&= \sup_{\gamma \colon X \to \{\mathtt{Acc},\mathtt{Rej}\} } \Delta^\varepsilon_\{\mathtt{Acc},\mathtt{Rej}\} ( \gamma(\mu_1) || \gamma(\mu_2))\\
&=  \sup_{\gamma \colon X \to \{\mathtt{Acc},\mathtt{Rej}\} }  \sup_{A \subseteq \{\mathtt{Acc},\mathtt{Rej}\}} (\Pr[\gamma(\mu_1) \in A] - e^\varepsilon \Pr[\gamma(\mu_2) \in A])\\
&=  \sup_{\gamma \colon X \to \{\mathtt{Acc},\mathtt{Rej}\} }  \sup_{A \subseteq \{\mathtt{Acc},\mathtt{Rej}\}} (\Pr[\mu_1 \in \inverse{\gamma}(A)] - e^\varepsilon \Pr[\mu_2 \in \inverse{\gamma}(A)])\\
& \stareq \sup_{S \subseteq X} (\Pr[\mu_1 \in S] - e^\varepsilon \Pr[\mu_2 \in S])\\
& = \Delta^\varepsilon_X(\mu_1 || \mu_2)
\end{align*}
We have the $2$-generatedness: $\overline{\Delta^\varepsilon}^{2} = \Delta^\varepsilon$.
The equality $(*)$ is proved as follows: for given $\gamma$ and $A$, we take $S = \inverse{\gamma}$.     Conversely, for any $S \subseteq X$ we take $A = \{\mathtt{Acc}\}$ and $\gamma = \chi_S$, which is the indicator function of $S$ defined by $\chi_S(x) = 1$ if $x \in S$ and $\chi_S(x) = 0$ otherwise.
\end{proof}

\paragraph{General version}
We can extend this result to general measurable setting by using the continuity of $\Delta^\varepsilon$ (see also \cite{1705001_2006}), which is obtained by $f$-divergence characterization of $\Delta^\varepsilon$ \cite{BartheOlmedo2013}.
For general measurable sapce $X$ and every pair $\mu_1,\mu_2 \in \prob(X)$ we have
\begin{align*}
\Delta^\varepsilon_X(\mu_1||\mu_2)
&= \sup_{\gamma \colon X \to \mathbb{N}}\Delta^\varepsilon_\mathbb{N}(\gamma(\mu_1)||\gamma(\mu_2))\\
&= \sup_{\gamma \colon X \to \mathbb{N}}\sup_{\gamma'\colon \mathbb{N} \to \prob(\{\mathtt{Acc},\mathtt{Rej}\}) }\Delta^\varepsilon_{\{\mathtt{Acc},\mathtt{Rej}\}}((\gamma' \bullet \gamma)(\mu_1)||(\gamma' \bullet \gamma)(\mu_2))\\
&= \sup_{\gamma \colon X \to \{\mathtt{Acc},\mathtt{Rej}\} }\Delta^\varepsilon_{\{\mathtt{Acc},\mathtt{Rej}\}}(\gamma(\mu_1)||\gamma(\mu_2))
\end{align*}
Functions are assumed to be measurable.

\subsection{Counterexample: R\'enyi-divergence is not $2$-generated}
\begin{theorem}
There are $\mu_1,\mu_2 \in {\tt Prob}(\{a,b,c\})$ such that
\[
\overline{D^{\alpha}}^2_{\{a,b,c\}} (\mu_1||\mu_2)< D^{\alpha}_{\{a,b,c\}}
(\mu_1||\mu_2) 
\]  
\end{theorem}
\begin{proof}
Let $p = (1/2)^{\beta/(\alpha-1)}$ and $\alpha+1 < \beta$ and define 
\begin{align*}
\mu_1 &= \frac{1}{3}\mathbf{d}_a + \frac{1}{3}\mathbf{d}_b + \frac{1}{3}\mathbf{d}_c,\\
\mu_2 &= \frac{p^2}{p^2+p+1}\mathbf{d}_a + \frac{p}{p^2+p+1}\mathbf{d}_b + \frac{1}{p^2+p+1}\mathbf{d}_c
\end{align*}
Since R\'enyi divergence is quasi-convex and satisfies data-processing inequality, it suffices to show the
proper inequality $D^\alpha_{\{\mathtt{Acc},\mathtt{Rej}\}}(\gamma(\mu_1)||\gamma(\mu_2))<D^\alpha_{\{a,b,c\}}(\mu_1||\mu_2)$ holds for any \emph{deterministic decision rule} $\gamma \colon \{a,b,c\} \to \{\mathtt{Acc},\mathtt{Rej}\}$.
There are $8$ cases of $\gamma \colon \{a,b,c\} \to \{\mathtt{Acc},\mathtt{Rej}\}$, but thanks to the data-processing inequality and reflexivity of R\'enyi divergence,
it suffices to consider $3$ cases: $(\gamma(a),\gamma(b),\gamma(c)) = (\mathtt{Acc},\mathtt{Acc},\mathtt{Rej}), (\mathtt{Acc},\mathtt{Rej},\mathtt{Acc}), (\mathtt{Rej},\mathtt{Acc},\mathtt{Acc})$.  
Hence,
\begin{align*}
&\frac{\exp((\alpha - 1)D^\alpha_{ \{a,b,c\}}(\mu_1||\mu_2))}{\exp((\alpha - 1)D^\alpha_{\{\mathtt{Acc},\mathtt{Rej}\}}(\gamma(\mu_1)||\gamma(\mu_2))}\\
&\geq \min \left(
\frac{
p^{2(1-\alpha)} + p^{1-\alpha} + 1
}{2^\alpha(p^2 + p)^{1-\alpha} + 1
},
\frac{
p^{2(1-\alpha)} + p^{1-\alpha} + 1
}{2^\alpha(p^2 + 1)^{1-\alpha} + p^{1-\alpha}
},
\frac{
p^{2(1-\alpha)} + p^{1-\alpha} + 1
}{2^\alpha(p + 1)^{1-\alpha} + p^{2(1-\alpha)}
}
\right)\\
&\geq \min \left(
\frac{
2^{\beta}+2^{-\beta}+1
}{
2^{\alpha}(p+1)^{1-\alpha}+2^{-\beta}
},
\frac{
2^{\beta}+2^{-\beta}+1
}{
2^{\alpha-\beta}(p^2+1)^{1-\alpha} + 1
},
\frac{
2^{\beta}+2^{-\beta}+1
}{
2^{\beta}+2^{\alpha-\beta}(p+1)^{1-\alpha}
}
\right)\\
&\geq \min \left(
\frac{
2^{\beta}+2^{-\beta}+1
}{
2^{\alpha+1}
},
\frac{
2^{\beta}+2^{-\beta}+1
}{
2^{\beta} + 1
}
\right) > 1.
\end{align*}
Hence,
\begin{align*}
&D^\alpha_{\{\mathtt{Acc},\mathtt{Rej}\}}(\gamma(\mu_1)||\gamma(\mu_2)) + \frac{1}{\alpha - 1} \log \min(\frac{2^\beta + 2^{-\beta} + 1}{2^{\alpha + 1}},\frac{2^\beta + 2^{-\beta} + 1}{2^\beta + 1})\\
&\leq D^{\alpha}_{\{a,b,c\}}(\mu_1||\mu_2) .
\end{align*}
holds for any $\gamma \colon \{a,b,c\} \to \{\mathtt{Acc},\mathtt{Rej}\}$.
By the data-processing inequality of R\'enyi divergence, this discussion does not depend on the choice of $\{\mathtt{Acc},\mathtt{Rej}\}$.
By weak Birkhoff-von Neumann theorem, and the quasi-convexity R\'enyi divergence, we conclude 
\[
\overline{D^{\alpha}}^2_X(\mu_1||\mu_2) + \frac{1}{\alpha - 1} \log \min(\frac{2^\beta + 2^{-\beta} + 1}{2^{\alpha + 1}},\frac{2^\beta + 2^{-\beta} + 1}{2^\beta + 1})\leq D^{\alpha}_{X}(\mu_1||\mu_2).
\]
\end{proof}

\subsection{Proof of $\infty$-generatedness of R\'enyi-divergence}
$f$-divergences is a class of divergences that are characterized by 
convex functions. 
For a given convex function $f \colon [0,\infty) \to \mathbb{R}$
satisfying $\lim_{t \to 0+}t f (0/t) = 0$ (this function is called weight function), 
we define an $f$-divergence $\Delta^{f}$ corresponding the function $f$, 
\[
\Delta^{f}_X(\mu_1||\mu_2) \defeq \sum_{x\in X} \mu_2(x) f\left( \frac{\mu_1(x)}{\mu_2(x)} \right).
\]
The $\alpha$-R\'enyi divergence $D^\alpha$ can also be characterized using $f$-divergence as follows:
\[
D^\alpha(\mu_1||\mu_2) = \frac{1}{\alpha - 1} \log \sum_{x\in X} \mu_2(x) \left( \frac{\mu_1(x)}{\mu_2(x)} \right)^\alpha
= \frac{1}{\alpha - 1} \log \Delta^{t \mapsto t^\alpha}_X(\mu_1||\mu_2).
\]
Remark that every $f$-divergence is quasi-convex (moreover jointly convex) and continuous, and satisfies data-processing inequality (see also \cite[Theorems 14--16]{1705001_2006}).

Since the mapping $t \mapsto \frac{1}{\alpha - 1} \log t$ is monotone, every $\alpha$-R\'enyi divergence $D^\alpha$ is also quasi-convex and satisfies data-processing inequality.
Thanks to the data-processing inequality,  every $\alpha$-R\'enyi divergence $D^\alpha$ is at least $\infty$-generated.
We need to prove that for every finite $k$, every $\alpha$-R\'enyi divergence $D^\alpha$ is not $k$-generated. 
To prove this, we use that the mapping $t \mapsto t^\alpha$ is strictly convex.

\begin{lemma}
If a weight function is strictly convex, its $f$-divergence $\Delta^{f}$ is not $k$-generated for every finite $k$.
\end{lemma}
\begin{proof}
Without loss of generality, we may assume $k > 1$.

Consider a pair $\mu_1,\mu_2 \in \prob(\{0,1,2,\ldots,k\})$ satisfying $\supp(\mu_1) = \supp(\mu_2) = \{0,1,2,\ldots,k\}$ and
$\mu_1(i)/\mu_2(i) \neq \mu_1(j)/\mu_2(j)$ where $1 \leq i,j \leq k+1$ and $i \neq j$.
We can give such distributions.
Then we obtain,
\begin{align*}
&\overline{\Delta^f}^{k}_{\{0,1,2,\ldots,k\}}(\mu_1 || \mu_2)\\
&=
\sup_{\gamma \colon \{0,1,2,\ldots,k\} \to \prob(\{0,1,2,\ldots,k-1\})}\Delta^f_{\{0,1,2,\ldots,k-1\}}(\gamma(\mu_1)||\gamma(\mu_2))
\\
&\quad \{ \text{Weak Birkhoff-von Neumann theorem and the joint convexity of } \Delta^f \}
\\
&= \max_{\gamma \colon \{0,1,2,\ldots,k\} \to \{0,1,2,\ldots,k-1\}}\Delta^f(\gamma(\mu_1)||\gamma(\mu_2))\\
& = \max_{\gamma \colon \{0,1,2,\ldots,k\} \to \{0,1,2,\ldots,k-1\}} \sum_{j = 0}^{k-1} f\left(\frac{\sum_{\gamma(i) = j} \mu_1(i)}{\sum_{\gamma(i) = j}\mu_2(i)}\right)(\textstyle \sum_{\gamma(i) = j} \mu_2(i))
\\
&\quad \{ \text{Jensen's inequality with the strict convexity of the weight function } f \}
\\
& < \sum_{i = 0}^{k} f\left(\frac{\mu_1(i)}{\mu_2(i)}\right) \mu_2(i) = \Delta^f(\mu_1 || \mu_2).
\end{align*}
Since $k+1 > k$, by Dirichlet's pigeonhole principle, for any $\gamma \colon \{0,1,2,\ldots,k\} \to \{0,1,2,\ldots,k-1\}$, for some $j \in \{0,1,2,\ldots,k\}$, there are at least two different $i_1, i_2 \in \{0,1,2,\ldots,k-1\}$ such that $\gamma(i_1) = j$ and $\gamma(i_2) = j$.
From the assumption on $\mu_1$ and $\mu_2$,
we have $(\mu_1(i_1)/\mu_2(i_1)) \neq (\mu_1(i_2)/\mu_2(i_2)) $
Since the function $f$ is strictly convex, by the condition for equality of Jensen's inequality, we have the strict inequality
\[
f \left(\frac{\mu_1(i_1) + \mu_1(i_2)}{\mu_2(i_1) + \mu_2(i_2)}\right) (\mu_2(i_1) + \mu_2(i_2)) 
<
f \left(\frac{\mu_1(i_1)}{\mu_2(i_1)}\right) \mu_2(i_1) + f \left(\frac{\mu_1(i_2)}{\mu_2(i_2)}\right) \mu_2(i_2).
\]
Therefore, for any $\gamma \colon \{0,1,2,\ldots,k\} \to \{0,1,2,\ldots,k-1\}$, we have
\[
\sum_{j = 1}^k \left(\frac{\sum_{\gamma(i) = j} \mu_1(i)}{\sum_{\gamma(i) = j}\mu_2(i)}\right)^\alpha ({\textstyle \sum_{\gamma(i) = j}} \mu_2(i))
< \sum_{i = 1}^{k+1} \left(\frac{\mu_1(i)}{\mu_2(i)}\right)^\alpha \mu_2(i).
\]
Since there only finite case of $\gamma \colon \{0,1,2,\ldots,k\} \to \{0,1,2,\ldots,k-1\}$,
we conclude $\overline{\Delta^f}^{k}_{\{0,1,2,\ldots,k\}}(\mu_1 || \mu_2) < \Delta^f_{\{0,1,2,\ldots,k\}}(\mu_1||\mu_2)$. 
Since every $f$-divergence satisfies data-processing inequality, this discussion does not depend on the choice of set $Y$ with $|Y| = k$ in the construction of the $k$-cut $\overline{\Delta^f}^{k}$.
Thus, $\Delta^f$ is not $k$-generated for any finite $k$.
\end{proof}

Since the mapping $t \mapsto \frac{1}{\alpha - 1} \log t$ is strict, we conclude,
\begin{corollary}
For any $alpha > 1$, the $\alpha$-R\'enyi divergence $D^\alpha$ is not $k$-generated for every finite $k$.
\end{corollary}

\subsection{Proof of Theorem 18}
\begin{theorem}[Theorem 18]
Let $\mu_1,\mu_2 \in \prob (X)$.
$\overline{\Delta}^2_X( \mu_1 || \mu_2) \leq \rho$ holds if and only if
for any $\gamma \colon X \to \prob (\{\mathtt{Acc},\mathtt{Rej}\})$,
\[
(\Pr[\gamma(\mu_1) = \mathtt{Rej}],\Pr[\gamma(\mu_2) = \mathtt{Acc}]) \in R^{\Delta}(\rho).
\]
\end{theorem}
\begin{proof}
We fix a $2$-cut $\overline{\Delta}^2$ of a divergence $\Delta$.
Suppose that it is defined with a set $W$ satisfying $|W| = 2$.
\[
\overline{\Delta}^2_X(\mu_1||\mu_2) = \sup_{\gamma \colon X \to \prob (W)}\Delta_W(\gamma (\mu_1)|| \gamma(\mu_2)).
\]
We recall the definition of privacy region
\[
R^{\Delta}(\rho)
=
\Set{(x,y)|
\overline{\Delta}^2_{\{\mathtt{Acc},\mathtt{Rej}\}}((1-x)\mathbf{d}_{\mathtt{Acc}} + x\mathbf{d}_{\mathtt{Rej}} || y\mathbf{d}_{\mathtt{Acc}}+(1-y)\mathbf{d}_{\mathtt{Rej}}) \leq \rho
}.
\]
Since every probability distribution $\nu \in\prob(\{\mathtt{Acc},\mathtt{Rej}\})$
can be rewritten as $\nu = \Pr[\nu = \mathtt{Acc}]\mathbf{d}_{\mathtt{Acc}} + \Pr[\nu = \mathtt{Rej}]\mathbf{d}_{\mathtt{Rej}}$, 
we obtain
\begin{align*}
&\overline{\Delta}^2_{\{\mathtt{Acc},\mathtt{Rej}\}}(\gamma(\mu_1)||\gamma(\mu_2)) \leq \rho\\
&\qquad\qquad \iff (\Pr[\gamma(\mu_1) = \mathtt{Rej}],\Pr[\gamma(\mu_2) = \mathtt{Acc}]) \in R^{\Delta}(\rho).
\end{align*}
Hence, it suffices to show 
\begin{align*}
\lefteqn{(\overline{\Delta}^2_X( \mu_1 || \mu_2) \leq \rho) }\\
&\iff
\forall
\gamma \colon X \to \prob (\{\mathtt{Acc},\mathtt{Rej}\}).
(\overline{\Delta}^2_{\{\mathtt{Acc},\mathtt{Rej}\}}(\gamma(\mu_1)||\gamma(\mu_2)) \leq \rho)
\end{align*}

($\implies$) 
Obvious by the data-processing inequality of the $2$-cut $\overline{\Delta}^2$.

($\impliedby$)
The assumption is equivalent to 
\begin{align*}
&\forall
\gamma \colon X \to \prob (\{\mathtt{Acc},\mathtt{Rej}\}).
\forall
\gamma' \colon \{\mathtt{Acc},\mathtt{Rej}\} \to \prob (W).\\
&\qquad\qquad (\Delta_W(\gamma' (\gamma(\mu_1))|| \gamma'(\gamma(\mu_2))) \leq \rho)
\end{align*}
Since $|W| = |\{\mathtt{Acc},\mathtt{Rej}\}| = 2$, this is equivalent to
\[
\gamma'' \colon X \to \prob (W).
\Delta_W(\gamma''(\mu_1)|| \gamma''(\mu_2)) \leq \rho.
\]
For any $\gamma \colon X \to \prob (\{\mathtt{Acc},\mathtt{Rej}\}).$ and $\gamma' \colon \{\mathtt{Acc},\mathtt{Rej}\} \to \prob (W).$ we take $\gamma'' = \gamma' \bullet \gamma$.
Conversely for any $\gamma'' \colon X \to \prob (W)$ we take $\gamma = f \bullet \gamma''$ and $\gamma' = \inverse{f}$ where $f \colon \{\mathtt{Acc},\mathtt{Rej}\} \to W$ is a bijection.  
\end{proof}

\subsection{Proof of Theorem 23 in general setting}

If the quasi-convex function $F \colon [0, 1]^{2k} \to [0, \infty]$ is also continuous, we can extend Theorem 23 to general measurable setting.
\begin{theorem}[$k$-generatedness in general setting]
Assume that $F \colon [0,1]^{2k} \to [0,\infty]$ is quasi-convex and continuous.
For any measurable space $X$, we have 
\[
\Delta_X (\mu_1,\mu_2) = \sup_{\substack{\gamma \colon X \to \prob(\{1,2,\ldots k \}) \\ \text{measurable function}}} \Delta_X (\gamma(\mu_1),\gamma(\mu_2)).
\]
\end{theorem}

\begin{proof}
We easily calculate as follows (functions are assumed to be measurable):
\begin{align*}
\lefteqn{\Delta_X (\mu_1,\mu_2)}\\
& = \sup \Set{ F(\mu_1(A_1),\cdots,\mu_1(A_k), \mu_2(A_1),\cdots,\mu_2(A_k) ) | \{ A_i \}_{i = 1}^k \colon \text{ m'ble partition of } X }\\
& = \sup\Set{ F(\mu_1(f^{-1}(1)),\cdots,\mu_1(f^{-1}(k)),
\mu_2(f^{-1}(1)),\cdots,\mu_1(f^{-1}(k)) ) | f \colon X \to \{1,2,\ldots,k\}}\\
& = \sup\Set{ F((f(\mu_1))(1),\cdots,(f(\mu_1))(k),
(f(\mu_2))(1),\cdots,(f(\mu_2))(k) ) | f \colon X \to \{1,2,\ldots,k\} }\\
& \leq \sup\Set{ F((\gamma(\mu_1))(1),\cdots,(\gamma(\mu_1))(k),
(\gamma(\mu_2))(1),\cdots,(\gamma(\mu_2))(k) ) | \gamma \colon
X \to \prob(\{1,2,\ldots,k\}) }\\
&\leq
\sup \Set{ F\left(
\begin{aligned}
&(\gamma(\mu_1))(A_1),\cdots,(\gamma(\mu_1))(A_k),\\
& \qquad (\gamma(\mu_2))(A_1),\cdots,(\gamma(\mu_2))(A_k)
\end{aligned}
\right) |
\begin{aligned}
&\gamma \colon X \to \prob(\{1,2,\ldots,k\}),\\
& \{A_i\}_{i = 1}^k \colon \text{ m'ble partition of } X
\end{aligned}
}\\
& = 
\sup_{ \gamma \colon X \to \prob(\{1,2,\ldots,k\})} \Delta_\{1,2,\ldots,k\}(\gamma(\mu_1),\gamma(\mu_2)) 
\end{align*}
Note that we treat $\{1,2,\ldots,k\}$ as a finite discrete space.
Consider the family $\{J_n\}_{n = 1}^\infty$ of finite sets (discrete spaces) defined as follows:
\[
J_n = \Set{(j_1,\ldots,j_{k}) | j_1,\ldots,j_{k} \in \{0,1,\ldots, 2^n -1\}, C^n_{j_1 \ldots j_{k}} \neq \emptyset}.
\]

We fix a measurable function $\gamma \colon X \to \prob(k)$ and treat
$\prob(k)$ as a subset of $[0,1]^k$.  For each $n \in \mathbb{N}$, we
define a measurable partition $\{C^n_{j_1 \ldots j_{k}}\}_{j_1,\ldots,j_{k} \in
\{0,1,\ldots, 2^n -1\}}$ of $X$ by
\begin{align*}
  C^n_{j_1 \ldots j_{k}} &= \gamma^{-1}(B^n_{j_1 \ldots j_k}) \\
& \qquad \text{ where } B^n_{j_1 \ldots j_N} = D_{j_1} \times \cdots \times D_{j_k} \qquad ((j_1 \ldots j_{k}) \in J_n),\\
& \qquad\qquad \quad D^n_0 = \{0\} \text{ and } D^n_{l+1} = \left({l}/{2^n}, {(l+1)}/{2^n} \right] \qquad (l = 0,1,2,\ldots,2^n-1).
\end{align*}
We next define $m^\ast_n \colon X \to J_n$ and $m_n \colon J_n \to X$ as follows:
$m^\ast_n(x)$ is the unique element $(j_1,\ldots,j_k) \in J_n$ satisfying $x \in C^n_{j_1,\ldots,j_k}$, and 
we choose $m_n(j_1,\ldots,j_k)$ is an element of $C^n_{j_1,\ldots,j_k}$.
Thanks to the measurability of each $C^n_{j_1 \ldots j_{k}}$,
the function $m^\ast_n$ is measurable, and the measurability of $m_n$ follows from the discreteness of $J_n$.
From the construction of $\{C^n_{j_1 \ldots j_{k}}\}_{j_1,\ldots,j_{k} \in \{0,1,\ldots, 2^n -1\}}$,
for any $n \in \mathbb{N}$, $x \in X$, and $i \in I$,  we have,
\[
|\gamma(x)(i) - (\gamma \circ m_n \circ m^\ast_n)(x)(i)| \leq 1/2^n
\]
This implies that the sequence $\{ \gamma \circ m_n \circ m^\ast_n \}_{n = 1}^{\infty}$
of measurable function converges uniformly to $\gamma$.
Hence, for any $n \in \mathbb{N}$ and $D \subseteq k$, we have 
\[
\left| \int \gamma(x)(D) ~d\mu_1(x) - \int (\gamma \circ m_n \circ m^\ast_n)(x)(D) ~d\mu_1(x) \right| \leq 1/2^n
\]
Hence the sequence of probability measures
$\{ (\gamma \circ m_n \circ m^\ast_n)(\mu_1) \}_{n = 1}^{\infty}$
converges to the probability measure $\gamma(\mu_1)$.
Similarly, $\{ (\gamma \circ m_n \circ m^\ast_n)(\mu_2) \}_{n = 1}^{\infty}$
converges to $\gamma(\mu_2)$.

By the continuity of $F$, we obtain
\begin{align*}
\lefteqn{ 
 F((\gamma(\mu_1))(A_1),\cdots,(\gamma(\mu_1))(A_k), (\gamma(\mu_2))(A_1),\cdots,(\gamma(\mu_2))(A_k) )
}\\
&=\lim_{n \to \infty}
F\left(
\begin{aligned}
&((\gamma \circ m_n \circ m^\ast_n)(\mu_1))(A_1),\cdots,((\gamma \circ m_n \circ m^\ast_n)(\mu_1))(A_k), \\
&\qquad ((\gamma \circ m_n \circ m^\ast_n)(\mu_2))(A_1),\cdots,((\gamma \circ m_n \circ m^\ast_n)(\mu_2))(A_k)
\end{aligned}
\right)\\
&=\lim_{n \to \infty}
F\left(
\begin{aligned}
&((\gamma \circ m_n)(m^\ast_n(\mu_1)))(A_1),\cdots,((\gamma \circ m_n)(m^\ast_n(\mu_1)))(A_k), \\
&\qquad ((\gamma \circ m_n)(m^\ast_n(\mu_2)))(A_1),\cdots,((\gamma \circ m_n)(m^\ast_n(\mu_2)))(A_k)
\end{aligned}
\right)\\
&\leq 
\sup_{n\in\mathbb{N}}
\Delta_{\{1,2,\ldots,k\}}( ((\gamma \circ m_n)(m^\ast_n(\mu_1))),((\gamma \circ m_n)(m^\ast_n(\mu_2))))\\
&\qquad \{ \text{Since ${J_n}$ is finite (countable and discrete), we can apply Theorem 23.} \} \\
&\leq \sup_{n\in\mathbb{N}}
\Delta_{J_n}(m^\ast_n(\mu_1),m^\ast_n(\mu_2))\\
&= \sup_{n\in\mathbb{N}}
\sup
\Set{
F( \left(\begin{array}{l@{}} f (m^\ast_n(\mu_1)))(1),\cdots,(f (m^\ast_n(\mu_1)))(k),\\\qquad (f (m^\ast_n(\mu_2)))(1),\cdots,((f (m^\ast_n(\mu_1)))(k) \end{array} \right)
|
f \colon J_n \to \{1,2,\ldots,k\} 
}\\
&
\leq
\sup
\{
F( (g(\mu_1))(1),\cdots,(g(\mu_1))(k), (g(\mu_2))(1),\cdots,(g(\mu_2))(k) )
\mid
g \colon X \to \{1,2,\ldots,k\} 
\}\\
&= \Delta_X (\mu_1,\mu_2).
\end{align*}
This implies $\sup_{ \gamma \colon X \to \prob(k)} \Delta_k(\gamma(\mu_1),\gamma(\mu_2)) \leq  \Delta_X (\mu_1,\mu_2)$.
\end{proof}

\section{Additional Results}
\subsection{Total variation distance is $2$-generated }
We recall the definition of the total variation distance 
\[
\mathtt{TV}_X(\mu_1 || \mu_2)
=\sup_{S \subseteq X} | \Pr[\mu_1 \in S] - \Pr[\mu_2 \in S] |.
\]
In a similar way as $\varepsilon$-divergence $\Delta^\varepsilon$, we can prove $2$-generatedness of  the total variation distance $\mathtt{TV}$, but we can prove it easily by applying Theorems 16--17 (Theorem 23 in the paper).

Define $F \colon [0,1]^4 \to [0,\infty]$ by $F(x,x',y,y') = |x - y|$. 
It is easy to check that the function is obviously quasi-convex,
and that we have $\mathtt{TV} = \Delta^F$.

\subsection{An optimal conversion law from Hellinger to DP}
We recall the definition of the Hellinger distance 
\[
\mathrm{HD}_X(\mu_1 || \mu_2) = 1 -\sum_{x \in X}\sqrt{\mu_1(x)\mu_2(x)}.
\]
Since it is the $f$-divergence of weight function $w(t) = \sqrt{t} - 1$ (strict convex),
the Hellinger distance is exactly $\infty$-generated, quasi-convex and continuous. 

Here is the essense of an optimal conversion law from the Hellinger distance to DP.
\begin{lemma}
We have $R^{\mathrm{HD}}(\rho) \subseteq R^{\Delta^{\varepsilon}}(\delta(\varepsilon,\rho))$ where
\begin{align}
\label{eq:conv:HDtoDP}
\delta(\varepsilon,\rho)&= 1 - t - \frac{f(t)}{g(t)}\\
\label{eq:conv:HDtoDP:tee}
t &= \frac{z^2 + 4 - z \sqrt{z^2 + 4}}{2(z^2 + 4)} \\
z &= \frac{1/e^\varepsilon -2(1-\rho) + 1}{(1-\rho)\sqrt{\rho(2-\rho)}} \notag\\
f(x) &= (1 - \rho)^2 (1-2x) + x - 2(1-\rho)\sqrt{d(2-d)x(1-x)} \notag\\
g(x) &= \frac{df}{dx}(x) = (1 - \rho)^2 (1-2x) + xf~ -2(1-\rho)\sqrt{d(2-d)x(1-x)} \notag
\end{align}
\end{lemma}
\begin{proof}
We may regard 
\begin{align*}
R^{\mathrm{HD}}(\rho)
&= \Set{(x,y) \in [0,1]^2 | 1 - \sqrt{x(1-y)} - \sqrt{(1-x)y} \leq \rho },\\
 R^{\Delta^{\varepsilon}}(\delta)
&= \Set{(x,y) \in [0,1]^2 | \max((1-x) - e^\varepsilon y, x - e^\varepsilon (1-y) ) \leq \delta }.
\end{align*}
We first calculate the boundary of $R^{\mathrm{HD}}(\rho)$.
Thus, we solve the following equation for $y$:
\[
1 - \sqrt{x(1-y)} - \sqrt{(1-x)y} = \rho.
\]
We first have
\begin{align*}
\lefteqn{1 - \sqrt{x(1-y)} - \sqrt{(1-x)y} = \rho}\\
&\iff (1-\rho)^2 - x (1-y) - y (1-x) = 2\sqrt{x(1-x)y(1-y)}\\
&\iff (1-\rho)^4 + x^2 (1-y)^2 + y^2 (1-x)^2 - 2x(1-y)(1-\rho)^2 - 2y(1-x)(1-\rho)^2
\end{align*}
The degree of this equation is $2$, so we can solve it. For given $x \in [0,1]$, we have 
\[
y = (1-\rho)^2(1-2x) + x \pm 2(1-\rho)\sqrt{x(1-x)\rho(2-\rho)}.
\]
Thanks to the Symmetry of $R^{\mathrm{HD}}(\rho)$ and $R^{\Delta^{\varepsilon}}(\delta)$,
we may consider the curve: 
\[
y = (1-\rho)^2(1-2x) + x - 2(1-\rho)\sqrt{x(1-x)\rho(2-\rho)} = f(x).
\]
The tangent of the curve $y = f(x)$ that passes the point $(t,f(t))$ is given by the equation 
$x - \frac{y}{g(t)} = t - \frac{f(t)}{g(t)}$ where $g(x) = \frac{df}{dx}(x)$.
We next find $t$ and $\delta$ that the lower boundary
\[
(1 - x) - e^\varepsilon y = \delta \iff x + e^\varepsilon y = 1 - \delta
\]
of $R^{\Delta^{\varepsilon}}(\delta(\varepsilon,\rho))$
is the same as the line $x - \frac{y}{g(t)} = t - \frac{f(t)}{g(t)}$.
We solve the equation $e^\varepsilon = \frac{1}{g(t)}$ on $t$ about the slope as (\ref{eq:conv:HDtoDP:tee}).
Finally, we obtain $\delta$ as (\ref{eq:conv:HDtoDP}).
\end{proof}
\begin{figure}
  \centering
  \includegraphics[width=7cm]{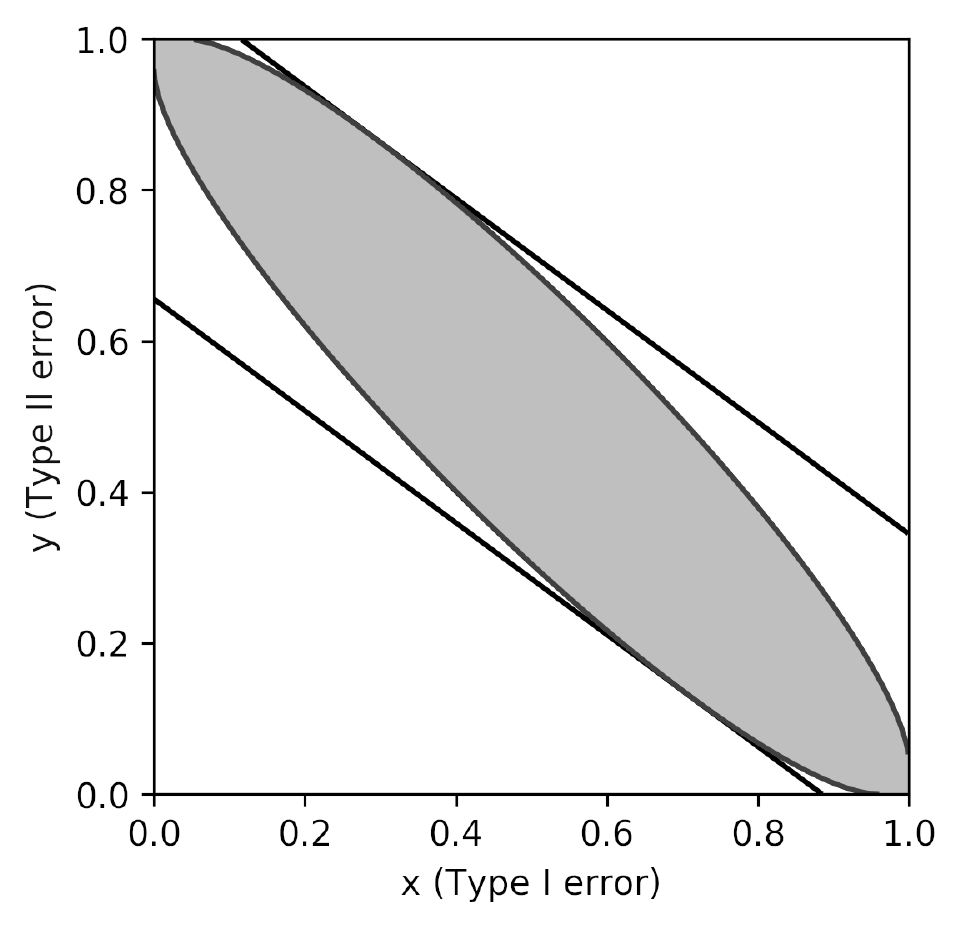}
  \caption{Comparison of the privacy region for DP and the one for 2-cut of Hellinger distance.}
\end{figure}

We conclude an optimal conversion law from the Hellinger distance to DP.
\begin{theorem}
We always have $\mathrm{HD}_X(d_1,d_2) \leq \rho \implies \Delta^{\varepsilon}_X(d_1,d_2) \leq \delta(\varepsilon,\rho)$ where $\delta(\varepsilon,\rho)$ is given by (\ref{eq:conv:HDtoDP}).
\end{theorem}

\bibliography{header,reference}
\end{document}